\documentclass[twoside,11pt]{article}
\usepackage{jair, theapa, rawfonts}
\usepackage{xcolor,soul,framed} 
\usepackage{xspace}
\usepackage{overpic}
\usepackage{enumitem}
\usepackage{graphicx}
\usepackage{amsfonts}
\usepackage{amssymb}
\usepackage{amsthm}
\usepackage{amsmath}
\usepackage{microtype}
\usepackage{url}
\usepackage[T1]{fontenc}

\usepackage[linesnumbered,ruled,vlined]{algorithm2e}
\usepackage{algpseudocode}

\newcommand{\rtp}{\textsc{GRP}\xspace}
\newcommand{\rta}{\textsc{GRA}\xspace}
\newcommand{\mpp}{\textsc{MAPF}\xspace}

\newcommand{\eecbs}{\textsc{EECBS}\xspace}

\newcommand{\ddm}{\textsc{DDM}\xspace}
\newcommand{\lacam}{\textsc{LaCAM}\xspace}
\newcommand{\rtpk}{\textsc{GRPKD}\xspace}
\newcommand{\rtptwo}{\textsc{GRP2D}\xspace}
\newcommand{\rtatwo}{\textsc{GRA2D}\xspace}
\newcommand{\rtpthree}{\textsc{GRP3D}\xspace}
\newcommand{\rtathree}{\textsc{GRA3D}\xspace}
\newcommand{\rthddd}{\textsc{GRH3D}\xspace}

\newcommand{\rtlmddd}{\textsc{GRLM3D}\xspace}
\newcommand{\rtmddd}{\textsc{GRM3D}\xspace}
\newcommand{\rtmdd}{\textsc{GRM}\xspace}
\newcommand{\rth}{\textsc{GRH}\xspace}
\newcommand{\mcp}{\textsc{MCP}\xspace}
\newcommand{\rtlm}{\textsc{GRLM}\xspace}

\newcommand{\rthlba}{\textsc{GRH-LBA}\xspace}
\newcommand{\rthip}{\textsc{GRH-IP}\xspace}
\newcommand{\rthpr}{\textsc{GRH-PR}\xspace}

\newtheorem{problem}{Problem}
\newtheorem{proposition}{Proposition}[section]
\newtheorem{lemma}{Lemma}[section]

\newtheorem{corollary}{Corollary}[section]
\newtheorem{theorem}{Theorem}[section]
\newtheorem{remark}{Remark}

\def\rtmapf{\textsc{GRM}\xspace}
\def\gtl{\mathcal G_{tl}\xspace}
\def\gbr{\mathcal G_{br}\xspace}

\newif\ifdraft
\drafttrue
\draftfalse

\ifdraft
\usepackage[paperheight=11in,paperwidth=9.5in,
			left=1.25in,right=1.25in,
			top=0.75in,bottom=0.75in,
			heightrounded,marginparwidth=1.2in,
			marginparsep=0.05in]{geometry}
\usepackage{xcolor}
\usepackage{xargs} 
\usepackage[textsize=footnotesize]{todonotes}
\newcommandx{\sh}[2][1=]{\todo[linecolor=blue,
			backgroundcolor=blue!10,bordercolor=blue,#1]{Han: #2}}
\newcommandx{\tg}[2][1=]{\todo[linecolor=orange,
			backgroundcolor=orange!10,bordercolor=orange,#1]{Greaten: #2}}
\newcommandx{\jy}[2][1=]{\todo[linecolor=green,
			backgroundcolor=green!10,bordercolor=green,#1]{JJ: #2}}
\else
\newcommand{\sh}[1]{{}}
\newcommand{\tg}[1]{{}}
\newcommand{\jy}[1]{{}}
\fi

\ShortHeadings{Expected $1.x$-Makespan-Optimal MAPF on Grids in Low-Poly Time}
{Teng Guo \& Jingjin Yu}

\begin{document}

\title{Expected $1.x$ Makespan-Optimal Multi-Agent Path Finding on Grid Graphs in Low Polynomial Time}

\author{\name Teng Guo \email teng.guo@rutgers.edu \\
       \name Jingjin Yu \email jingjin.yu@rutgers.edu \\
       \addr Rutgers, the State University of New Jersey, Piscataway, NJ, USA.}


\maketitle

\begin{abstract}
Multi-Agent Path Finding (\mpp) is NP-hard to solve optimally, even on graphs, suggesting no polynomial-time algorithms can compute exact optimal solutions for them. 
This raises a natural question: How optimal can polynomial-time algorithms reach? 
Whereas algorithms for computing constant-factor optimal solutions have been developed, the constant factor is generally very large, limiting their application potential. 
In this work, among other breakthroughs, we propose the first low-polynomial-time \mpp algorithms delivering $1$-$1.5$ (resp., $1$-$1.67$)  asymptotic makespan optimality guarantees for 2D (resp., 3D) grids for random instances at a very high $1/3$ agent density, with high probability. Moreover, when regularly distributed obstacles are introduced, our methods experience no performance degradation. These methods generalize to support $100\%$ agent density.  

Regardless of the dimensionality and density, our high-quality methods are enabled by a unique hierarchical integration of two key building blocks. At the higher level, we apply the labeled Grid Rearrangement Algorithm (\rta), capable of performing efficient reconfiguration on grids through row/column shuffles. At the lower level, we devise novel methods that efficiently simulate row/column shuffles returned by \rta. 
Our implementations of \rta-based algorithms are highly effective in extensive numerical evaluations, demonstrating excellent scalability compared to other SOTA methods. For example, in 3D settings, \rta-based algorithms readily scale to grids with over $370,000$ vertices and over $120,000$ agents and consistently achieve conservative makespan optimality approaching $1.5$, as predicted by our theoretical analysis. 

Full source code of our \rta implementations will be made available at

\url{https://github.com/arc-l/rubik-table}

\end{abstract}
We examine multi-agent pathfinding (\mpp \cite{stern2019multi}) on two- and three-dimensional grids with potentially regularly distributed obstacles (see Fig.~\ref{fig:jd_center}).
The main objective of \mpp is to find collision-free paths for routing many agents from a start configuration to a goal configuration.  
In practice, solution optimality is of key importance, yet optimally solving \mpp is NP-hard~\cite{surynek2010optimization,yu2013structure}, even in planar settings \cite{yu2015intractability} and on obstacle-less grids~\cite{demaine2019coordinated}. 
\mpp algorithms apply to a diverse set of practical scenarios, including formation reconfiguration~\cite{PodSuk04}, object transportation~\cite{RusDonJen95}, swarm robotics\cite{preiss2017crazyswarm,honig2018trajectory}, to list a few.
Even when constrained to grid-like settings, \mpp algorithms still find many large-scale applications, including warehouse automation for general order fulfillment \cite{wurman2008coordinating}, grocery order fulfillment \cite{mason2019developing}, and parcel sorting \cite{wan2018lifelong}. 

 \begin{figure}[htbp]
        \centering
        \begin{overpic}               
        [width=1\linewidth]{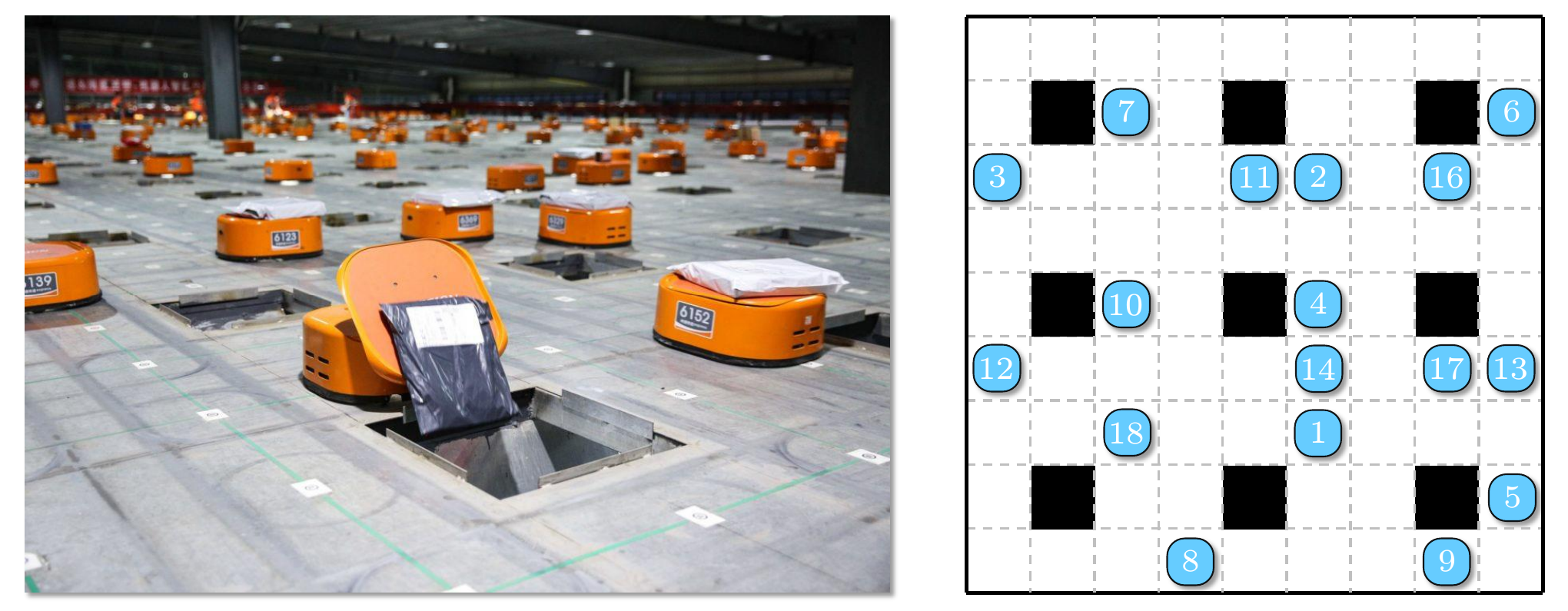}
             \footnotesize
             \put(28.5, -3) {(a)}
             \put(78.5, -3) {(b)}
        \end{overpic}
        \caption{ (a) Real-world parcel sorting system (at JD.com) using many agents on a large grid-like environment with holes for dropping parcels; (b) A snapshot of a similar \mpp instance our methods can solve in polynomial-time with provable optimality guarantees. In practice, our algorithms scale to 2D maps of size over {$450 \times 300$}, supporting over $45$K agents and achieving better than $1.5$-optimality (see, e.g.,  Fig.~\ref{fig:revise_third}). They scale further in 3D settings. 
        } 
        \label{fig:jd_center}
    \end{figure}

Motivated by applications including order fulfillment and parcel sorting, we focus on grid-like settings with extremely high agent density, i.e., with $\frac{1}{3}$ or more graph vertices occupied by agents. Whereas recent studies \cite{yu2018constant,demaine2019coordinated} show such problems can be solved in polynomial time yielding constant-factor optimal solutions, the constant factor is prohibitively high ($\gg 1$) to be practical.
In this research, we tear down this \mpp planning barrier, achieving $(1+\delta)$-makespan optimality asymptotically with high probability in polynomial time where $\delta \in (0, 0.5]$ for 2D grids and $\delta\in (0,\frac{2}{3}]$ for 3D grids.

This study's primary theoretical and algorithmic contributions are outlined below and summarized in Table~\ref{tab: algorithm_summarize}. Through judicious applications of a novel row/column-shuffle-based global Grid Rearrangement (GR) method called the Rubik Table Algorithm \cite{szegedy2023rubik}\footnote{It was noted in \cite{szegedy2023rubik} that their Rubik Table Algorithm can be applied to solve \mpp; we do not claim this as a contribution of our work. Rather, our contribution is to dramatically bring down the constant factor through a combination of meticulous algorithm design and careful analysis.}, employing many classical algorithmic techniques, and combined with careful analysis, we establish that in \emph{low polynomial} time:
\begin{itemize}[leftmargin=3.5mm]
    \setlength\itemsep{1mm}
    \item For $m_1m_2$ agents on a 2D $m_1\times m_2$ grid, $m_1 \ge m_2$, i.e., at maximum agent density, \rtmapf (Grid Rearrangement for \mpp) computes a solution for an arbitrary \mpp instance under a makespan of $4m_1 + 8m_2$; For $m_1m_2m_3$ agents on a 3D $m_1\times m_2 \times m_3$ grids, $m_1 \ge m_2 \ge m_3$, \rtmapf computes a solution under a makespan of $4m_1+8m_2+8m_3$; 
    \item For $\frac{1}{3}$  agent density with uniformly randomly distributed start/goal configurations, \rth (Grid Rearrangement with Highways) computes a solution with a makespan of $m_1 + 2m_2 + o(m_1)$ on 2D grids and $m_1+2m_2+2m_3+o(m_1)$ on 3D grids, with high probability. In contrast, such an instance has a minimum makespan of $m_1 + m_2 - o(m_1)$ for 2D grids and $m_1+m_2+m_3-o(m_1)$ for 3D grids with high probability. This implies that, as $m_1 \to \infty$, an asymptotic optimality guarantee of $\frac{m_1 + 2m_2}{m_1+m_2} \in (1, 1.5]$ is achieved for 2D grids and $\frac{m_1+2m_2+2m_3}{m_1+m_2+m_3}\in(1,\frac{5}{3}]$ for 3D grids, with high probability;
     \item For $\frac{1}{2}$ agent density, the same optimality guarantees as in the $\frac{1}{3}$ density setting can be achieved using \rtlm (Grid Rearrangements with Line Merge), using a merge-sort inspired agent routing heuristic; 
    \item Without modification, \rth achieves the same $\frac{m_1 + 2m_2}{m_1+m_2}$ optimality guarantee on 2D grids with up to $\frac{m_1m_2}{9}$ regularly distributed obstacles together with $\frac{2m_1m_2}{9}$ agents (e.g., Fig.~\ref{fig:jd_center}(b)). Similar guarantees hold for 3D settings; 
    \item For agent density up to $\frac{1}{2}$, for an arbitrary (i.e., not necessarily random) instance, a solution can be computed with a makespan of $3m_1+4m_2+o(m_1)$ on 2D grids and $3m_1+4m_2+4m_3+o(m_1)$ on 3D grids. 
    \item With minor numerical updates to the relevant guarantees, the above-mentioned results also generalize to grids of arbitrary dimension $k \ge 4$. 
\end{itemize}

 Moreover, we have developed many effective and principled heuristics that work together with the \rta-based algorithms to further reduce the computed makespan by a significant margin. These heuristics include (1) An optimized matching scheme, required in the application of the 
 grid rearrangement algorithm,
 based on linear bottleneck assignment (LBA), (2) A polynomial time path refinement method for compacting the solution paths to improve their quality.
As demonstrated through extensive evaluations, our methods are highly scalable and capable of tackling instances with tens of thousands of densely packed agents. Simultaneously, the achieved optimality matches/exceeds our theoretical prediction. With the much-enhanced scalability, our approach unveils a promising direction toward the development of practical, provably optimal multi-agent routing algorithms that run in low polynomial time. 

\begin{table}[h]
\scriptsize
  \centering
  \begin{tabular}{|c|c|c|c|c|}
    \hline
    \textbf{Algorithms} & \rtmapf 2x3 &  \rtmapf 4x2 & \rtlm &\rth  \\
\hline
\textbf{Applicable Density} & $\leq 1$ &$\leq 1$ &$\leq \frac{1}{2}$ & $\leq \frac{1}{3}$ \\
\hline
    \textbf{Makespan Upperbound} & $7(2m_2+m_1)$ & $4(2m_2+m_1)$ & $4m_2+3m_1$ & $4m_2+3m_1$ \\
    \hline
    \textbf{Asymptotic Makespan}&$7(2m_2+m_1)$ &$4(2m_2+m_1)$ &$2m_2+m_1+o(m_1)$ & $2m_2+m_1+o(m_1)$\\
    \hline
      \textbf{Asymptotic Optimality}&$7(1+\frac{m_2}{m_1+m_2})$ &$4(1+\frac{m_2}{m_1+m_2})$ &$1+\frac{m_2}{m_1+m_2}$ & $1+\frac{m_2}{m_1+m_2}$\\
    \hline
  \end{tabular}
  \caption{
Summary results of the algorithms proposed in this work. All algorithms operate on grids of dimensions $m_1\times m_2$. \rth, \rtmapf, and \rtlm are all further derived from \rtatwo, each a variant characterized by distinct low-level shuffle movements. Specifically for \rtmapf,  the low-level shuffle movement is tailored for facilitating full-density robot movement}
  \label{tab: algorithm_summarize}
\end{table}

This paper builds on two conference publications~\cite{guo2022sub,Guo2022PolynomialTN}. Besides providing a unified treatment of the problem that streamlined the description of \rta-based algorithms under 2D/3D/$k$D settings for journal archiving, the manuscript further introduces many new results, including (1) The baseline \rtmapf method (for the full-density case) is significantly improved with a much stronger makespan optimality guarantee, by a factor of $\frac{7}{4}$; (2) A new and general polynomial-time path refinement technique is developed that significantly boosts the optimality of the plans generated by all \rta-based algorithms; (3) Complete and substantially refined proofs are provided for all theoretical developments in the paper; and (4) The evaluation section is fully revamped to reflect the updated theoretical and algorithmic development. 
%
%

{\textbf{Related work.}} 
Literature on multi-agent path and motion planning \cite{hopcroft1984complexity,ErdLoz86} is expansive; here, we mainly focus on graph-theoretic (i.e., the state space is discrete) studies \cite{yu2016optimal,stern2019multi}. As such, in this paper, \mpp refers explicitly to graph-based multi-agent path planning. 
Whereas the feasibility question has long been positively answered  \cite{KorMilSpi84}, the same cannot be said for finding optimal solutions, as computing time- or distance-optimal solutions are NP-hard in many settings, including on general graphs \cite{goldreich2011finding,surynek2010optimization,yu2013structure}, planar graphs \cite{yu2015intractability,banfi2017intractability}, and regular grids \cite{demaine2019coordinated} that is similar to the setting studied in this work.  

Nevertheless, given its high utility, especially in e-commerce applications \cite{wurman2008coordinating,mason2019developing,wan2018lifelong} that are expected to grow significantly \cite{dekhne2019automation,covid-auto}, many algorithmic solutions have been proposed for optimally solving \mpp. 
Among these, combinatorial-search-based solvers \cite{lam2019branch} have been demonstrated to be fairly effective.  \mpp solvers may be classified as being optimal or suboptimal. 
Reduction-based optimal solvers solve the problem by reducing the \mpp problem to another problem, e.g., SAT~\cite{surynek2012towards}, answer set programming~\cite{erdem2013general}, integer linear programming (ILP)~\cite{yu2016optimal}.
Search-based optimal \mpp solvers include  EPEA* \cite{goldenberg2014enhanced}, ICTS \cite{sharon2013increasing}, CBS \cite{sharon2015conflict}, M* \cite{wagner2015subdimensional}, and many others. 

Due to the inherent intractability of optimal \mpp, optimal solvers usually exhibit limited scalability, leading to considerable interest in suboptimal solvers.
Unbounded solvers like push-and-swap~\cite{luna2011push}, push-and-rotate~\cite{de2014push}, windowed hierarchical cooperative A${}^*$~\cite{silver2005cooperative}, BIBOX~\cite{surynek2009novel}, all return feasible solutions very quickly, but at the cost of solution quality.
Balancing the running time and optimality is one of the most attractive topics in the study of \mpp. 
Some algorithms emphasize the scalability without sacrificing as much optimality, e.g., ECBS~\cite{barer2014suboptimal}, DDM \cite{han2020ddm}, PIBT \cite{Okumura2019PriorityIW}, PBS \cite{ma2019searching}. There are also learning-based solvers ~\cite{damani2021primal,sartoretti2019primal,li2021message} that scale well in sparse environments. Effective orthogonal heuristics have also been proposed \cite{guo2021spatial}. 
Recently, $O(1)$-approximate or constant factor time-optimal algorithms have been proposed, e.g.  \cite{yu2018constant,demaine2019coordinated,han2018sear}, that tackle highly dense instances. However, these algorithms only achieve a low-polynomial time guarantee at the expense of very large constant factors, rendering them theoretically interesting but impractical. 

In contrast, with high probability, our methods run in low polynomial time with provable $1.x$ asymptotic optimality. To our knowledge, this paper presents the first \mpp algorithms to simultaneously guarantee polynomial running time and $1.x$ solution optimality, which works for any dimension $\ge 2$.

\textbf{Organization.} The rest of the paper is organized as follows.
In Sec.~\ref{sec:pre}, starting with 2D grids, we provide a formal definition of graph-based \mpp, and introduce the Grid Rearrangement problem and the associated baseline algorithm (\rta) for solving it.
\rtmapf, a basic adaptation of \rta for \mpp at maximum agent density which ensures a makespan upper bound of $4m_1 + 8m_2$, is described in Sec.~\ref{sec:1:1}. An accompanying lower bound of $m_1 + m_2 - o(m_1)$ for random \mpp instances is also established. 
In Sec.~\ref{sec:1:2} we introduce \rth for $\frac{1}{3}$ agent density achieving a makespan upper bound of $m_1 + 2m_2 + o(m_1)$. Obstacle support is then discussed. 
We continue to show how $\frac{1}{2}$ agent density may be supported with similar optimality guarantees. 
In Sec.~\ref{sec: three_d}, we generalize the algorithms to work on 3$+$D grids.
In Sec.~\ref{sec:opt-boost}, we introduce multiple optimality-boosting heuristics to significantly improve the solution quality for all variants of Grid Rearrangement-based solvers.
We thoroughly evaluate the performance of our methods in Sec.~\ref{sec:eval} and conclude with Sec.~\ref{sec:conclusion}.

\section{Preliminaries}\label{sec:pre}
\subsection{Multi-Agent Path Finding on Graphs (\mpp)}
Consider an undirected graph $\mathcal G(V, E)$ and $n$ agents with start configuration $S = \{s_1, \dots, s_n\} \subseteq V$ and goal configuration $G = \{g_1, \dots, g_n\} \subseteq V$.
A {\em path} for agent $i$ is a map 
$P_i: \mathbb {N} \to V$ where $\mathbb N$ is the set of non-negative integers. 
A feasible $P_i$ must be a sequence of vertices that connects $s_i$ and $g_i$: 
1) $P_i(0) = s_i$;
2) $\exists T_i \in \mathbb N$, s.t. $\forall t \geq T_i, P_i(t) = g_i$;
3) $\forall t > 0$, $P_i(t) = P_i(t - 1)$ or $(P_i(t), P_i(t - 1)) \in E$. 
A path set $\{P_1, \ldots, P_n\}$ is feasible iff each $P_i$ is feasible and for all $t \ge 0$ and $1\le i < j \le n$, $P_i(t)=P$, it does not happen that: 1) $P_i(t) = P_j(t)$; 2) $P_i(t) = P_j(t+1) \wedge P_j(t) = P_i(t+1)$.

We work with $\mathcal G$ being $4$-connected grids in 2D and $6$-connected grids in 3D, aiming to mainly minimize the \emph{makespan}, i.e., $\max_i\{|{P}_i|\}$ (later, a sum-of-cost objective is also briefly examined).
Unless stated otherwise, $\mathcal G$ is assumed to be an $m_1 \times m_2$ grid with $m_1 \ge m_2 $ in 2D and $m_1\times m_2\times m_3$ grid with $m_1\ge m_2\ge m_3$ in 3D. As a note, ``randomness'' in this paper always refers to uniform randomness. 
The version of \mpp we study is sometimes referred to as \emph{one-shot} MAPF \cite{stern2019multi}. We mention our results also translate to guarantees on the life-long setting \cite{stern2019multi}, briefly discussed in Sec.~\ref{sec:conclusion}. 

\vspace{-2mm}
\subsection{The Grid Rearrangement Problem (\rtp)}
The Grid Rearrangement problem (\rtp) (first proposed in ~\cite{szegedy2023rubik} as the Rubik Table problem) formalizes the task of carrying out globally coordinated object reconfiguration operations on lattices, with many interesting applications. 
The problem has many variations; we start with the 2D form, to be generalized to higher dimensions later.

\begin{problem}[{\normalfont \bf Grid Rearrangement Problem in 2D (\rtptwo)} \cite{szegedy2023rubik}]
Let $M$ be an $m_1 (row) \times m_2 (column)$ table, $m_1 \ge m_2$, containing $m_1m_2$ items, one in each table cell. 
The items have $m_2$ colors with each color having a multiplicity of $m_1$.
In a \emph{shuffle} operation, the items in a single column or a single row of $M$ may be permuted in an arbitrary manner. 
Given an arbitrary configuration $X_I$ of the items, find a sequence of shuffles that take $M$ from $X_I$ to the configuration where
row $i$, $1 \leq i \leq m_1$, contains only items of color $i$. The problem may also be \emph{labeled}, i.e., each item has a unique label in $1, \ldots, m_1m_2$.
\end{problem}

A key result \cite{szegedy2023rubik}, which we denote here as the (labeled) Grid Rearrangement Algorithm in 2D (\rtatwo), shows that a colored \rtptwo can be solved using $m_2$ column shuffles followed by $m_1$ row shuffles, implying a low-polynomial time computation time.
Additional $m_1$ row shuffles then solve the labeled \rtptwo. We illustrate how \rtatwo works on an $m_1 \times m_2$ table with $m_1 =4$ and $m_2 =3$ (Fig.~\ref{fig:rubik}); we refer the readers to \cite{szegedy2023rubik} for details. The main supporting theoretical result is given in Theorem~\ref{t:rta}, which depends on Theorem~\ref{t:hall}.
%

\begin{theorem}[Hall’s Matching theorem with parallel edges \cite{hall2009representatives,szegedy2023rubik}]\label{t:hall}
Let $B$ be a $d$-regular ($d$ > 0) bipartite graph on $n + n$ nodes,
possibly with parallel edges. Then $B$ has a perfect
matching.
\end{theorem}

\begin{theorem}[{\normalfont \bf Grid Rearrangement Theorem \cite{szegedy2023rubik}}]\label{t:rta}
An arbitrary Grid Rearrangement problem on an $m_1\times m_2$ table can be solved using $m_1 + m_2$ shuffles. The labeled Grid Rearrangement problem can be solved using $2m_1 + m_2$ shuffles.
\end{theorem}

\rtatwo operates in two phases. In the first, a bipartite graph $B(T, R)$ is constructed based on the initial table where the bipartite set $T$ are the colors/types of items, and the set $R$ the rows of the table  (see Fig.~\ref{fig:rubik}(b)). An edge is added to $B$ between $t \in T$ and $r \in R$ for every item of color $t$ in row $r$. 
From $B(T,R)$, which is a \emph{regular bipartite graph}, $m_2$ \emph{perfect matchings} can be computed as guaranteed by Theorem~\ref{t:hall}. Each matching, containing $m_1$ edges, dictates how a table column should look like after the first phase. For example, the first set of matching (solid lines in Fig.~\ref{fig:rubik}(b)) says the first column should be ordered as yellow, cyan, red, and green, shown in Fig.~\ref{fig:rubik}(c). 
After all matchings are processed, we get an intermediate table, Fig.~\ref{fig:rubik}(c). Notice each row of Fig.~\ref{fig:rubik}(a) can be shuffled to yield the corresponding row of Fig.~\ref{fig:rubik}(c); a key novelty of \rtatwo.
After the first phase of $m_1$ row shuffles, the intermediate table (Fig.~\ref{fig:rubik}(c))  can be rearranged with $m_2$ column shuffles to solve the colored \rtptwo (Fig.~\ref{fig:rubik}(d)). Another $m_1$ row shuffles then solve the labeled \rtptwo (Fig.~\ref{fig:rubik}(e)). It is also possible to perform the rearrangement using $m_2$ column shuffles, followed by $m_1$ row shuffles, followed by another $m_2$ column shuffles. 

\begin{figure}[h]
        \centering
        \begin{overpic}[width=\linewidth]{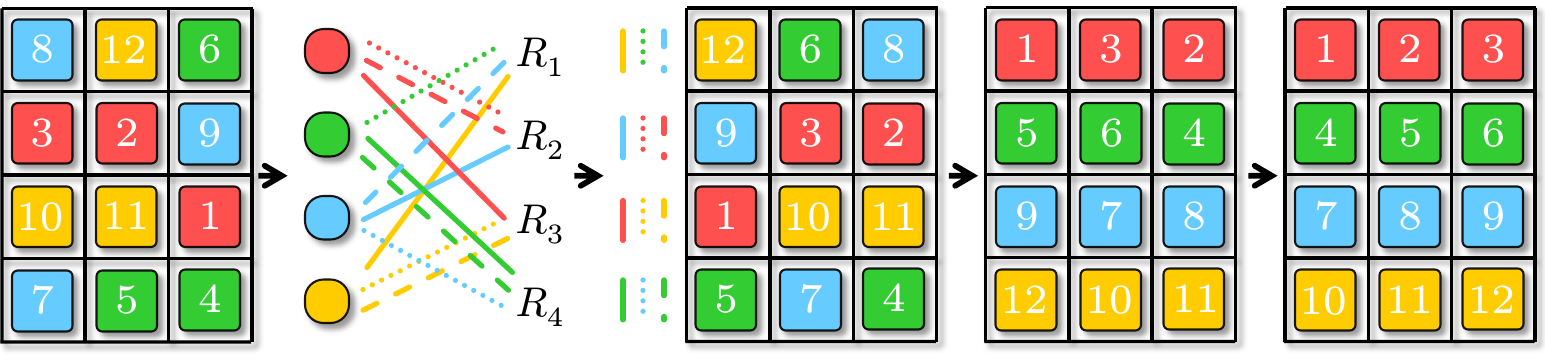}
             \footnotesize
             \put(6.5,  -2) {(a)}
             \put(26,  -2) {(b)}
             \put(50.5,  -2) {(c)}
             \put(70,  -2) {(d)}
             \put(89.5,  -2) {(e)}
        \end{overpic}
        
        \caption{Illustration of applying the \emph{$11$ shuffles}. (a) The initial $4\times 3$ table with  a random arrangement of 12 items that are colored and labeled. The labels are consistent with the colors. (b) The constructed bipartite graph. It contains $3$ perfect matchings, determining the $3$ columns in (c); only color matters in this phase. (c) Applying $4$ row shuffles to (a), according to the matching results, leads to an intermediate table where each column has one color appearing exactly once. (d) Applying $3$ column shuffles to (c) solves a colored \rtptwo. (e) $4$ additional row shuffles fully sort the labeled items.} 
        \label{fig:rubik}
    \end{figure}
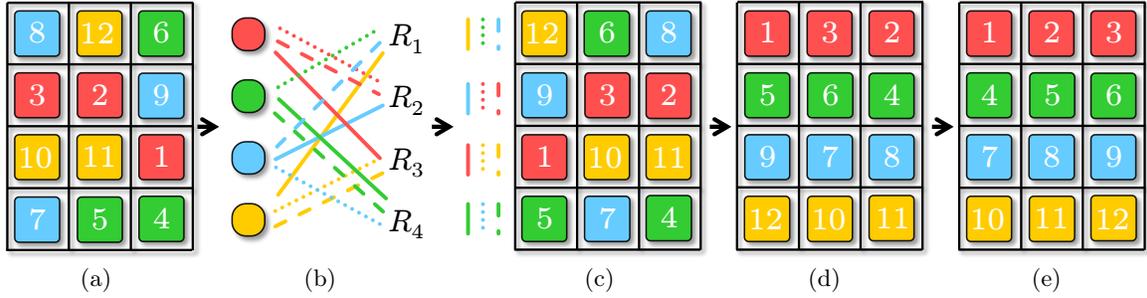

\rtatwo runs in $O(m_1m_2\log m_1)$ (notice that this is nearly linear with respect to $n = m_1m_2$, the total number of items) expected time or $O(m_1^2m_2)$ deterministic time. 

\section{Solving \mpp at Maximum Density Leveraging \rtatwo}\label{sec:1:1}
The built-in global coordination capability of \rtatwo naturally applies to solving makespan-optimal \mpp.
Since \rtatwo only requires \emph{three rounds} of shuffles and each round involves either parallel row shuffles or parallel column shuffles, if each round of shuffles can be realized with makespan proportional to the size of the row/column, then a makespan upper bound of $O(m_1 + m_2)$ can be guaranteed. 
This is in fact achievable even when all of $\mathcal G$'s vertices are occupied by agents, by recursively applying a \emph{labeled line shuffle algorithm} \cite{yu2018constant}, which can arbitrarily rearrange a line of $m$ agents embedded in a grid using $O(m)$ makespan. 
\begin{lemma}[Basic Simulated Labeled Line Shuffle \cite{yu2018constant}]\label{l:line-shuffle} For $m$ labeled agents on a straight path of length $m$, embedded in a 2D grid, they may be arbitrarily ordered in $O(m)$ steps. Moreover, multiple such reconfigurations can be performed on parallel paths within the grid. 
\end{lemma}

The key operation is based on a localized, $3$-step pair swapping routine, shown in Fig.~\ref{fig:figure8}. For more details on the line shuffle routine, see \cite{yu2018constant}. 

\begin{figure}[h!]
        \centering
        \includegraphics[width=\linewidth]{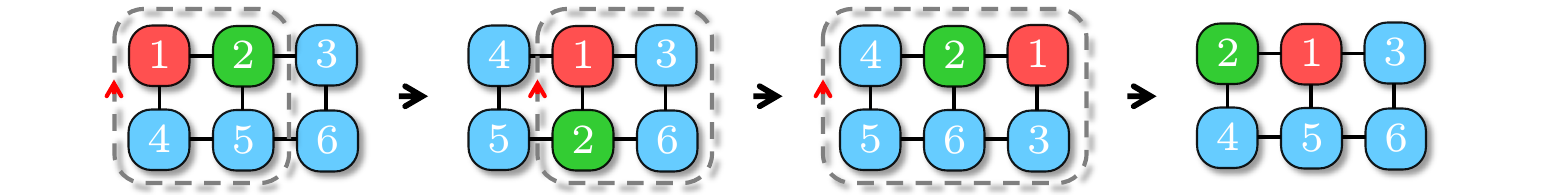}
\caption{On a $2 \times 3$ grid, swapping two agents may be performed in three steps with three cyclic rotations.} 
        \label{fig:figure8}
\end{figure}
    
However, the basic simulated labeled line-shuffle algorithm has a large constant factor. Each shuffle takes 3 steps; doing arbitrary shuffling of a $2 \times 3$ takes $20+$ steps in general. The constant factor further compounds as we coordinate the shuffles across multiple lines. Borrowing ideas from \emph{parallel odd-even sort} \cite{bitton1984taxonomy}, we can greatly reduce the constant factor in Lemma \ref{l:line-shuffle}. We will do this in several steps. First, we need the following lemma. By an \emph{arbitrary horizontal swap} on a grid, we mean an arbitrary reconfiguration or a grid row.

\begin{lemma}\label{l:reconfigure}
It takes at most $7$, $6$, $6$, $7$, $6$, and $8$ steps to perform arbitrary combinations of arbitrary horizontal swaps on $3\times 2$, $4\times 2$, $2 \times 3$, $3\times 3$, $2\times 4$, and $3\times 4$ grids, respectively.
\end{lemma}
\begin{proof}
Using integer linear programming \cite{yu2016optimal}, we exhaustively compute makespan-optimal solutions for arbitrary horizontal reconfigurations on $3\times 2$ ($8 = 2^3$ possible cases), $4 \times 2$ grids ($2^4$ possible cases), $2 \times 3$ grids ($6^2$ possible cases), $3 \times 3$ grids ($6^3$ possible cases), $2 \times 4$ grids ($24^2$ possible cases), and $3 \times 4$ grids ($24^3$ possible cases), which confirms the claim. 
\end{proof}
As an example, it takes seven steps to horizontally ``swap'' all three pairs of agents on a $3\times 2$ grid, as shown in Fig.~\ref{fig:six}. 

 \begin{figure}[!htbp]
        \centering
        \includegraphics[width=1\linewidth]{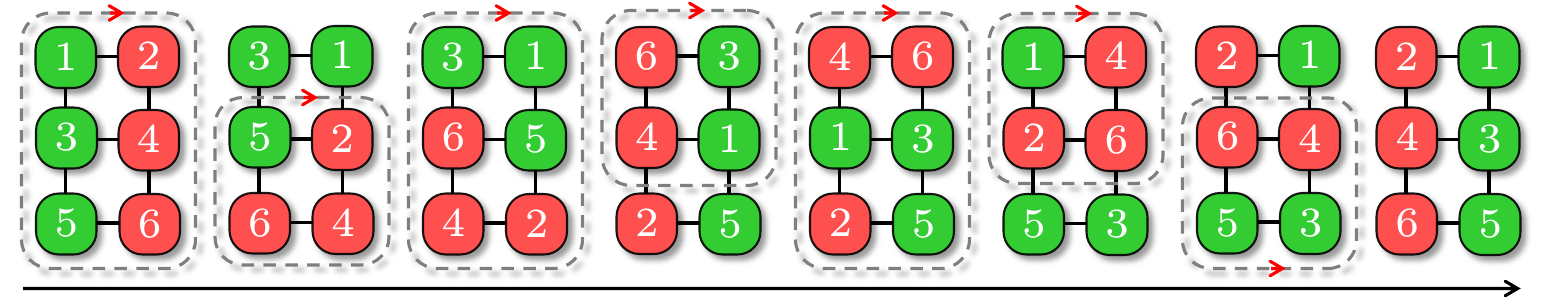}
\caption{An example of a full horizontal ``swap'' on a $3\times 2$ grid that takes seven steps, in which all three pairs are swapped. 
} 
        \label{fig:six}
    \end{figure}

\begin{lemma}[Fast Line Shuffle]\label{l:fast-line-shuffle}
Embedded in a 2D grid, $m$ agents on a straight path of length $m$ may be arbitrarily ordered in $7m$ steps. Moreover, multiple such reconfigurations can be executed in parallel within the grid. 
\end{lemma}
\begin{proof}
Arranging $m$ agents on a straight path of length $m$ may be realized using parallel odd-even sort \cite{bitton1984taxonomy} in $m-1$ rounds, which only requires the ability to simulate potential pairwise ``swaps'' interleaving odd phases (swapping agents located at positions $2k + 1$ and $2k + 2$ on the path for some $k$) and even phases (swapping agents located at positions $2k + 2$ and $2k + 3$ on the path for some $k$). Here, it does not matter whether $m$ is odd or even. 
To simulate these swaps, we can partition the grid embedding the path into $3 \times 2$ grids in two ways for the two phases, as illustrated in Fig.~\ref{fig:odd-even}. 
\begin{figure}[h!]
        \centering
        \includegraphics[width=1\linewidth]{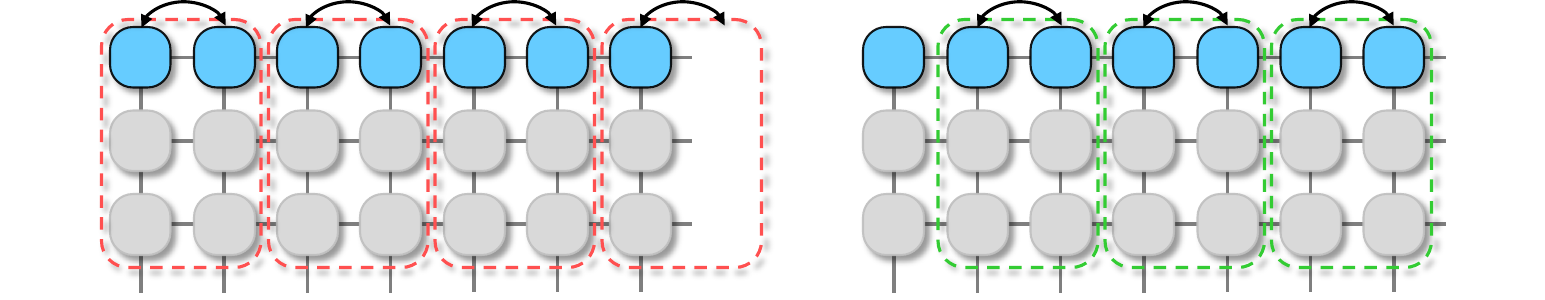}
\caption{Partitioning a grid into disjoint $3 \times 2$ grids in two ways for simulating odd-even sort. Highlighted agent pairs may be independently ``swapped'' within each $3\times 2$ grid as needed.} 
        \label{fig:odd-even}
    \end{figure}

A perfect partition requires that the second dimension of the grid, perpendicular to the straight path, be a multiple of $3$. If not, some partitions at the bottom can use $4 \times 2$ grids. By Lemma~\ref{l:reconfigure}, each odd-even sorting phase can be simulated using at most $7m$ steps. Shuffling on parallel paths is directly supported. 
\end{proof}

After introducing a $7m$ step line shuffle, we further show how it can be dropped to $4m$, using similar ideas. The difference is that an updated parallel odd-even sort will be used with different sub-grids of sizes different from $2\times 3$ and $2\times 4$. 

The updated parallel odd-even sort operates on blocks of \emph{four} (Fig.~\ref{fig:odd-even-2}) instead of \emph{two} (Fig.~\ref{fig:odd-even}), which cuts down the number of parallel sorting operations from $m-1$ to about $m/2$. 
Here, if $m$ is not even, a partition will leave either $1$ or $3$ at the end. For example, if $m= 11$, it can be partitioned as $4, 4, 3$ and $2, 4, 4, 1$ in the two parallel sorting phases.  

\begin{figure}[h]
        \centering
        \includegraphics[width=0.9\linewidth]{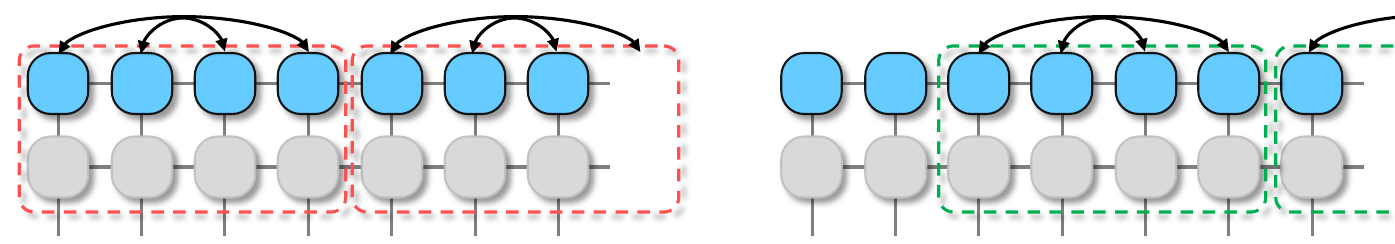}
\caption{We can make the parallel odd-even sort work twice faster by increasing the swap block size from two to four.} 
        \label{fig:odd-even-2}
    \end{figure}

With the updated parallel odd-even sort, we must be able to make swaps on blocks of four. We do this by partition an $m_1\times m_2$ grid into $2\times 4$, which may have leftover sub-grids of sizes $2\times 3$, $3\times 4$, and $3 \times 3$. 
Using the same reasoning in proving Lemma~\ref{l:fast-line-shuffle} and with Lemma~\ref{l:reconfigure}, we have

\begin{lemma}[Faster Line Shuffle]\label{l:faster-line-shuffle}
Embedded in a 2D grid, $m$ agents on a straight path of length $m$ may be arbitrarily ordered in approximately $4m$ steps. Moreover, multiple such reconfigurations can be executed in parallel within the grid. 
\end{lemma}
\begin{proof}[Proof sketch]
The updated parallel odd-even sort requires a total of $m/2$ steps. Since each step operated on a $2\times 4$, $2 \times 3$, $3 \times 4$, or $3\times 3$ grid, which takes at most $8$ steps, the total makespan is approximately $4m$. 
\end{proof}

Combining \rtatwo and fast line shuffle (Lemma~\ref{l:faster-line-shuffle}) yields a polynomial time \mpp algorithm for fully occupied grids with a makespan of $4m_1 + 8m_2$. 

\begin{theorem}[\mpp on Grids under Maximum Agent Density, Upper Bound]\label{t:rtm-makespan}
\mpp on an $m_1\times m_2$ grid, $m_1 \ge m_2 \ge 3$, with each grid vertex occupied by an agent, can be solved in polynomial time in a makespan of $4m_1 + 8m_2$.
\end{theorem}
\tg{added proof}
\begin{proof}[Proof Sketch]
By combining the \rtatwo and the line shuffle algorithms, the problem can be efficiently solved through two row-shuffle phases and one column-shuffle phase. During the row-shuffle phases, all rows can be shuffled in $4m_2$ steps, and similarly, during column-shuffle phases, all columns can be shuffled in $4m_1$ steps. Summing up these steps, the entire problem can be addressed in $4m_1 + 8m_2$ steps. The primary computational load lies in computing the perfect matchings, a task achievable in $O(m_1m_2\log m_1)$ expected time or $O(m_1^2m_2)$ deterministic time.
\end{proof}

It is clear that, by exploiting the idea further, smaller makespans can potentially be obtained for the full-density setting, but the computation required for extending  Lemma~\ref{l:reconfigure} will become more challenging. It took about two days to compute all $24^3$ cases for the $3\times 4$ grid, for example.

We call the resulting algorithm from the process \rtmdd (for 2D), with ``M'' denoting \emph{maximum density}, regardless of the used sub-grids. The straightforward pseudo-code is given in Alg.~\ref{alg:rubik}. The comments in the main \rtmdd routine indicate the corresponding \rtatwo phases. For \mpp on an $m_1 \times m_2$ grid with row-column coordinates $(x, y)$, we say agent $i$ belongs to color $1 \le j \le m_1$ if $g_i.y=j$.
Function $\texttt{Prepare()}$ in the first phase finds intermediate states $\{\tau_i\}$ for each agent through perfect matchings and routes them towards the intermediate states by (simulated) column shuffles. 
If the agent density is smaller than required, we may fill the table with ``virtual'' agents \cite{han2018sear,yu2018constant}.
For each agent $i$, we have $\tau_i.y=s_i.y$. 
Function $\texttt{ColumnFitting()}$ in the second phase routes the agents to their second intermediate states $\{\mu_i\}$ through row shuffles where $\mu_{i}.x=\tau_{i}.x$ and $\mu_i.y=g_i.y$. 
In the last phase, function $\texttt{RowFitting()}$ routes the agents to their final goal positions using additional column shuffles.

\begin{algorithm}
\begin{small}
\DontPrintSemicolon
\SetKwProg{Fn}{Function}{:}{}
\SetKwFunction{Fprepare}{Prepare}
\SetKwFunction{Fcolumnfitting}{ColumnFitting}
\SetKwFunction{Frowfitting}{RowFitting}
\SetKwFunction{FRtMapf}{RTM}

  \caption{Labeled Grid Rearrangement Based \mpp Algorithm for 2D (\rtatwo)\label{alg:rubik}}
  \KwIn{Start and goal vertices $S=\{s_i\}$ and $G=\{g_i\}$}
  \Fn{\textsc{RTA2D}({$S,G$})}{
$\texttt{Prepare}(S,G)$ \quad\quad\quad\hspace{-0.5mm} \Comment{Computing Fig.~\ref{fig:rubik}(b)}\;
$\texttt{ColumnFitting}(S,G)$ \Comment{Fig.~\ref{fig:rubik}(a) $\to$ Fig.~\ref{fig:rubik}(c)}\;
$\texttt{RowFitting}(S,G)$ \quad\quad\hspace{-1.5mm} \Comment{Fig.~\ref{fig:rubik}(c) $\to$
Fig.~\ref{fig:rubik}(d)}\;
}
\vspace{1mm}
  \Fn{\Fprepare{$S,G$}}{
  
    $A\leftarrow [1,...,m_1m_2]$\;
  \For{$(t,r)\in [1,...,m_1]\times[1,...,m_1]$}{
  \If{$\exists i\in A$ where $s_i.x=r\wedge g_i.y=t$}
  {
 add edge $(t,r)$ to $B(T,R)$\;
 remove $i$ from $A$ \;
  }
    }
 compute matchings $\mathcal{M}_1,...,\mathcal{M}_{m_2}$ of $B(T, R)$\;
 $A\leftarrow [1,...,m_1m_2]$\;
 \ForEach{$\mathcal{M}_r$ and $(t,r)\in \mathcal{M}_r$}{
 \If{$\exists i\in A$ where $s_i.x=r\wedge g_i.y=t$}{
 $\tau_i\leftarrow (r, s_i.y)$ and remove $i$ from $A$\;
  mark agent $i$ to go to $\tau_i$\;
 }
 }
 perform simulated column shuffles in parallel 
 }
 \vspace{1mm}
  \Fn{\Fcolumnfitting{$S,G$}}{
        \ForEach{$i\in [1,...,m_1m_2]$}{
        $\mu_i\leftarrow (\tau_i.x,g_i.y)$ and mark agent $i$ to go to $\mu_i$\;
        }
  perform simulated row shuffles in parallel      
  }
\vspace{1mm}
  \Fn{\Frowfitting{$S,G$}}{
        \ForEach{$i\in [1,...,m_1m_2]$}{
        mark agent $i$ to go to $g_i$\;
        }
  perform simulated column shuffles in parallel 
  }
\end{small}
\end{algorithm}

We now establish the optimality guarantee of \rtmapf on 2D grids, assuming \mpp instances are randomly generated. 
For this, a precise lower bound is needed. 

\begin{proposition}[Precise Makespan Lower Bound of \mpp on 2D Grids]\label{p:makespan-lower}
The minimum makespan of random \mpp instances on an $m_1 \times m_2$ grid with $\Theta(m_1m_2)$ agents is $m_1 + m_2 - o(m_1)$ with arbitrarily high probability as $m_1\to \infty$.
\end{proposition}
\begin{proof}
Without loss of generality, let the constant in $\Theta(m_1m_2)$ be some $c \in (0, 1]$, i.e., there are $cm_1m_2$ agents. 
We examine the top left and bottom right corners of the $m_1 \times m_2$ grid $\mathcal G$. In particular, let $\gtl$ (resp.,  $\gbr$) be the top left (resp., bottom right) $\alpha m_1\times \alpha m_2$ sub-grid of $\mathcal G$, for some positive constant $\alpha \ll 1$.
For $u \in V(\gtl)$ and $v \in V(\gbr)$, assuming each grid edge has unit distance, then the Manhattan distance between $u$ and $v$ is at least $(1-2\alpha)(m_1 + m_2)$. 
Now, the probability that some $u  \in V(\gtl)$ and $v \in V(\gbr)$ are the start and goal, respectively, for a single agent, is $\alpha^4$. For $cm_1m_2$ agents, the probability that at least one agent's start and goal fall into $\gtl$ and $\gbr$, respectively, is $p =1 - (1 - \alpha^4)^{cm_1m_2}$. 

Because $(1 - x)^y < e^{-xy}$ for $0 < x < 1$ and $y > 0$ \footnote{This is because $\log(1-x) < -x$ for $0 < x < 1$; 
multiplying both sides by a positive $y$ and exponentiate with base $e$ then yield the inequality.}, $p > 1 - e^{-\alpha^4cm_1m_2}$. 
Therefore, for arbitrarily small $\alpha$, we may choose $m_1$ such that $p$ is arbitrarily close to $1$. 
For example, we may let $\alpha = m_1^{-\frac{1}{8}}$, which decays to zero as $m_1 \to \infty$, then it holds that the makespan is $(1 - 2\alpha)(m_1 + m_2) = m_1 + m_2 - 2m_1^{-\frac{1}{8}}(m_1 + m_2) = m_1 + m_2 - o(m_1)$ with probability $p > 1 - e^{-c\sqrt{m_1}m_2}$. 
\end{proof}

Comparing the upper bound established in Theorem~\ref{t:rtm-makespan} and the lower bound from Proposition~\ref{p:makespan-lower} immediately yields

\begin{theorem}[Optimality Guarantee of \rtmapf]\label{t:rtmapf}
For random \mpp instances on an $m_1 \times m_2$ grid with $\Omega(m_1m_2)$ agents, as $m_1\to \infty$, \rtmapf computes in polynomial time solutions that are $4(1 + \frac{m_2}{m_1+m_2})$-makespan optimal, with high probability.
\end{theorem}
\tg{added proof sketch}
\begin{proof}[Proof Sketch]
By Proposition~\ref{p:asymptotic_lowerbound} and Theorem~\ref{t:rtm-makespan}, the asymptotic optimality ratio is $4 \left(1 + \frac{m_2}{m_1+m_2}\right)$.
\end{proof}

 \rtmapf always runs in polynomial time and has the same running time as \rtatwo; the high probability guarantee only concerns solution optimality. 
 The same is true for other high-probability algorithms proposed in this paper. We also note that high-probability guarantees imply guarantees in expectation. 

\section{Near-Optimally Solving \mpp with up to One-Third and One-Half Agent Densities}\label{sec:1:2}
Though \rtmapf runs in polynomial time and provides constant factor makespan optimality in expectation, the constant factor is $4+$ due to the extreme density. In practice, a agent density of around $\frac{1}{3}$ (i.e., $n = \frac{m_1m_2}{3}$) is already very high. As it turns out,  with $n = cm_1m_2$ for some constant $c > 0$ and $n \le \frac{m_1m_2}{2}$, the constant factor can be dropped to close to 1.  

\subsection{Up to One-Third Density: Shuffling with Highway Heuristics}\label{subsec:1:3}
For $\frac{1}{3}$ density, we work with random \mpp instances in this subsection; arbitrary instances for up to $\frac{1}{2}$ density are addressed in later subsections.
Let us assume for the moment that $m_1$ and $m_2$ are multiples of three; we partition $\mathcal G$ into $3\times 3$ cells (see, e.g., Fig.~\ref{fig:jd_center}(b) and Fig.~\ref{fig:example}).
We use Fig.~\ref{fig:example}, where Fig.~\ref{fig:example}(a) is a random start configuration and Fig.~\ref{fig:example}(f) is a random goal configuration, as an example to illustrate \rth -- Grid Rearrangement with Highways, targeting agent density up to $\frac{1}{3}$. 
\rth has two phases: \emph{unlabeled reconfiguration} and \emph{\mpp resolution with Grid Rearrangement and highway heuristics}. 

\begin{figure}[htbp]
\centering

        

              \begin{overpic}[width=\linewidth]{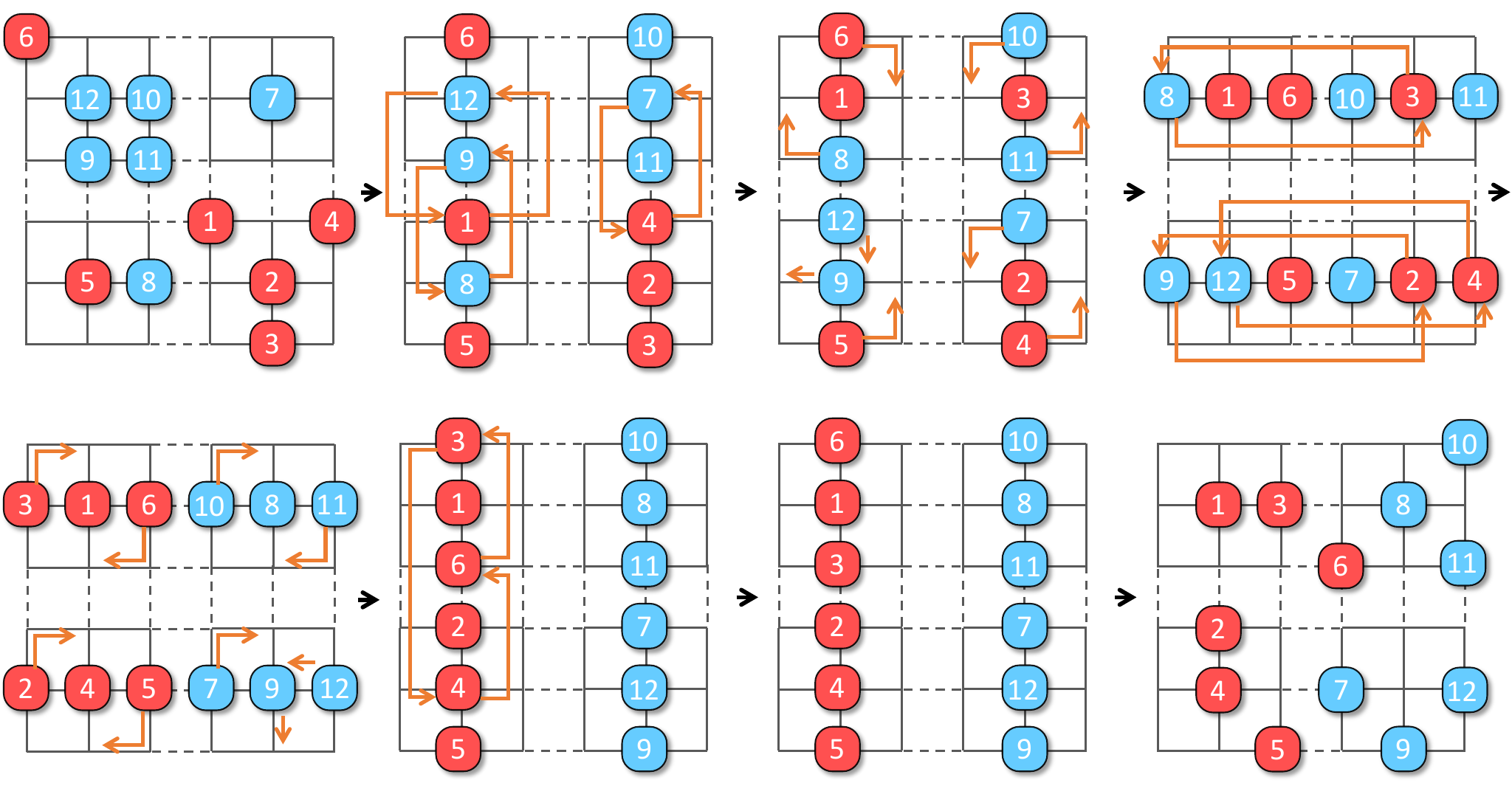}
             \footnotesize
             \put(10.5, 25) {(a)}
             \put(35.5, 25) {(b)}
             \put(60.5, 25) {(c)}
             \put(85.5, 25) {(d)}
             \put(10.5, -1) {(e)}
             \put(35.5, -1) {(f)}
             \put(60.5, -1) {(g)}
             \put(85.5, -1) {(h)}
        
        \end{overpic}
        \caption{An example of applying \rth to solve an \mpp instance. (a) The start configuration; (b) The start balanced configuration obtained from (a);  (c) The intermediate configuration obtained from the Grid Rearrangement preparation phase; (d). The intermediate configuration obtained from (c); (e) The intermediate configuration obtained from the column fitting phase; (f) Apply additional column shuffles for labeled items; (g) The goal balanced configuration obtained from the goal configuration; (h) The goal configuration.} 
        \label{fig:example}
    \end{figure}
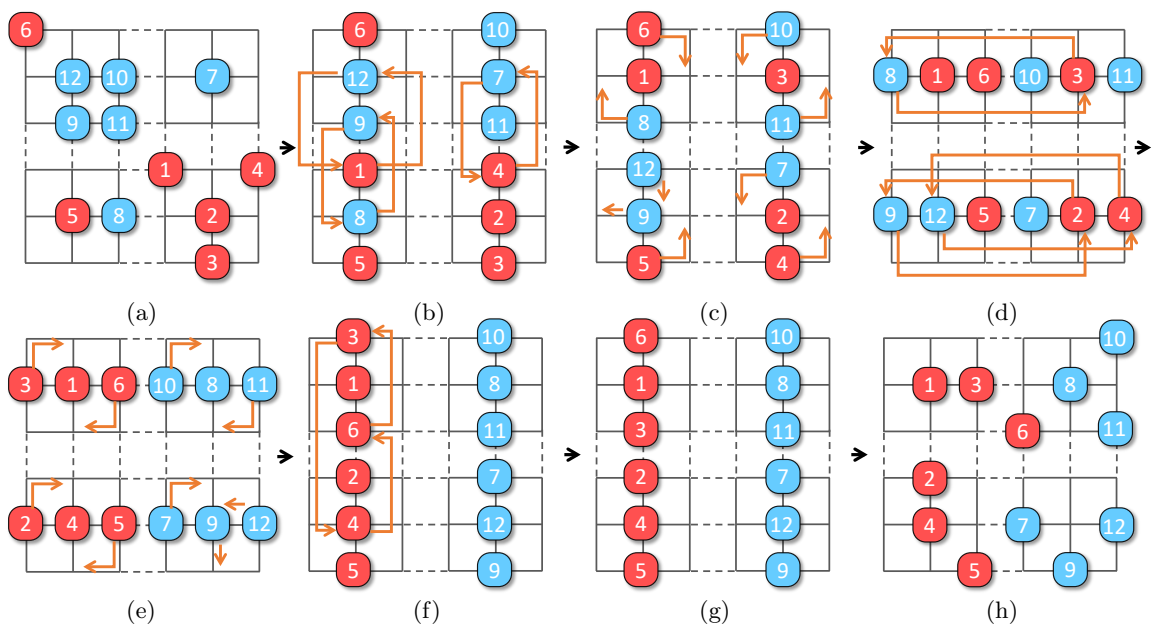

In unlabeled reconfiguration, agents are treated as being indistinguishable. Arbitrary start and goal configurations (under $\frac{1}{3}$ agent density) are converted to intermediate configurations where each $3 \times 3$ cell contains no more than $3$ agents. We call such configurations \emph{balanced}. With high probability, random \mpp instances are not far from being balanced. To establish this result (Proposition~\ref{p:phase:1}), we need the following. 

\begin{theorem}[Minimax Grid Matching \cite{leighton1989tight}]\label{t:minimax}
Consider an $m \times m$ square containing $m^2$ points following the uniform distribution. Let $\ell$ be the minimum length such that there exists a perfect matching of the $m^2$ points to the grid points in the square for which the distance between every pair of matched points is at most $\ell$. Then $\ell = O(\log^{\frac{3}{4}}m)$ with high probability.
\end{theorem}

Theorem~\ref{t:minimax} applies to rectangles with the longer side being $m$ as well (Theorem 3 in \cite{leighton1989tight}). 

\begin{proposition}\label{p:phase:1}
On an $m_1\times m_2$ grid, with high probability, a random configuration of $n = \frac{m_1m_2}{3}$ agents is of distance $o(m_1)$ to a balanced configuration. 
\end{proposition}
\begin{proof}
We prove for the case of $m_1 = m_2 = 3m$ using the minimax grid matching theorem (Theorem~\ref{t:minimax}); generalization to $m_1 \ge m_2$ can be then seen to hold using the generalized version of Theorem~\ref{t:minimax} that applies to rectangles (Theorem 3 of \cite{leighton1989tight}, which applies to arbitrarily simply-connected region within a square region).

Now let $m_1 = m_2 = 3m$. We may view a random configuration of $m^2$ agents on a $3m\times 3m$ grid as randomly placing $m^2$ continuous points in an $m\times m$ square with scaling (by three in each dimension) and rounding. 
By Theorem~\ref{t:minimax}, a random configuration of $m^2$ continuous points in an $m \times m$ square can be moved to the $m^2$ grid points at the center of the $m^2$ disjoint unit squares within the $m \times m$ square, where each point is moved by a distance no more than $O(\log^{\frac{3}{4}}m)$, with high probability. 
Translating this back to a $3m \times 3m$ gird, $m^2$ randomly distributed agents on the grid can be moved so that each $3\times 3$ cell contains exactly one agent, and the maximum distance moved for any agent is no more than $O(\log^{\frac{3}{4}}m)$, with high probability. Applying this argument three times yields that a random configuration of $\frac{m_1^2}{3}$ agents on an $m_1\times m_1$ gird can be moved so that each $3\times 3$ cell contains exactly three agents, and no agent needs to move more than a $O(\log^{\frac{3}{4}}{m_1})$ steps, with high probability. Because the agents are indistinguishable, overlaying three sets of reconfiguration paths will not cause an increase in the distance traveled by any agent. \end{proof}

In the example, unlabeled reconfiguration corresponds to Fig.~\ref{fig:example}(a)$\to$Fig.~\ref{fig:example}(b) and Fig.~\ref{fig:example}(f)$\to$Fig.~\ref{fig:example}(e) (\mpp solutions are time-reversible). 
We simulated the process of unlabeled reconfiguration for $m_1 = m_2 = 300$, i.e., on a $300 \times 300$ grids. For $\frac{1}{3}$ agent density, the actual number of steps averaged over $100$ random instances is less than $5$.
We call configurations like  Fig.~\ref{fig:example}(b)-(e), which have all agents concentrated vertically or horizontally in the middle of the $3\times 3$ cells, \emph{centered balanced} or simply \emph{centered}. 
Completing the first phase requires solving two unlabeled problems \cite{YuLav12CDC,Ma2016OptimalTA}, doable in polynomial time. 

In the second phase, \rta is applied with a highway heuristic to get us from Fig.~\ref{fig:example}(b) to Fig.~\ref{fig:example}(e), transforming between vertical centered configurations and horizontal centered configurations. 
To do so, \rta is applied (e.g., to Fig.~\ref{fig:example}(b) and (e)) to obtain two intermediate configurations (e.g., Fig.~\ref{fig:example}(c) and (d)).
To go between these configurations, e.g., Fig.~\ref{fig:example}(b)$\to$Fig.~\ref{fig:example}(c), we apply a heuristic by moving agents that need to be moved out of a $3\times 3$ cell to the two sides of the middle columns of Fig.~\ref{fig:example}(b), depending on their target direction. If we do this consistently, after moving agents out of the middle columns, we can move all agents to their desired goal $3\times 3$ cell without stopping nor collision. 
Once all agents are in the correct $3\times 3$ cells, we can convert the balanced configuration to a centered configuration in at most $3$ steps, which is necessary for carrying out the next simulated row/column shuffle. 
Adding things up, we can simulate a shuffle operation using no more than $m + 5$ steps where $m = m_1$ or $m_2$. 
The efficiently simulated shuffle leads to low makespan \mpp algorithms. It is clear that all operations take polynomial time; a precise running time is given at the end of this subsection.

\begin{theorem}[Makespan Upper Bound for Random \mpp, $\le \frac{1}{3}$ Density]\label{t:rtm-ramdom}
For random \mpp instances on an $m_1 \times m_2$ grid, where $m_1 \ge m_2$ are multiples of three, for $n \le \frac{m_1m_2}{3}$ agents, an $m_1 + 2m_2 + o(m_1)$ makespan solution can be computed in polynomial time, with high probability. 
\end{theorem}

\begin{proof}
By Proposition~\ref{p:phase:1}, unlabeled reconfiguration requires distance $o(m_1)$ with high probability. This implies that a plan can be obtained for unlabeled reconfiguration that requires $o(m_1)$ makespan (for detailed proof, see Theorem 1 from \cite{yu2018constant}).
For the second phase, by Theorem~\ref{t:rta}, we need to perform $m_1$ parallel row shuffles with a row width of $m_2$, followed by $m_2$ parallel column shuffles with a column width of $m_1$, followed by another $m_1$ parallel row shuffles with a row width of $m_2$. Simulating these shuffles require $m_1 + 2m_2 + O(1)$ steps. Altogether, a makespan of $m_1 + 2m_2 + o(m_1)$ is required, with a very high probability. 
\end{proof}

Contrasting Theorem~\ref{t:rtm-ramdom} and Proposition~\ref{p:makespan-lower} yields

\begin{theorem}[Makespan Optimality for Random \mpp, $\le \frac{1}{3}$ Density]\label{t:rth-ratio}
For random \mpp instances on an $m_1 \times m_2$ grid, where $m_1 \ge m_2$ are multiples of three, for $n = cm_1m_2$ agents with $c \le \frac{1}{3}$, as $m_1 \to \infty$, a $(1 + \frac{m_2}{m_1 + m_2})$ makespan optimal solution can be computed in polynomial time, with high probability. 
\end{theorem}

Since $m_1\ge m_2$, $1 + \frac{m_2}{m_1 + m_2} \in (1, 1.5]$. In other words, in polynomial running time, \rth achieves $(1 + \delta)$ asymptotic makespan optimality for $\delta \in (0, 0.5]$, with high probability. 

From the analysis so far, if $m_1$ and/or $m_2$ are not multiples of $3$, it is clear that all results in this subsection continue to hold for agent density $\frac{1}{3} - \frac{(m_1\bmod 3)(m_2\bmod 3)}{m_1m_2}$, which is arbitrarily close to $\frac{1}{3}$ for large $m_1$ and $m_2$. It is also clear that the same can be said for grids with certain patterns of regularly distributed obstacles (Fig.~\ref{fig:jd_center}(b)), i.e., 

\begin{corollary}[Random \mpp, $\frac{1}{9}$ Obstacle and $\frac{2}{9}$ Agent Density]
For random \mpp instances on an $m_1 \times m_2$ grid, where $m_1 \ge m_2$ are multiples of three and there is an obstacle at coordinates $(3k_1 + 2, 3k_2 + 2)$ for all applicable $k_1$ and $k_2$, for $n = cm_1m_2$ agents with $c \le \frac{2}{9}$, a solution can be computed in polynomial time that has makespan $m_1 + 2m_2 + o(m_1)$ with high probability.
As $m_1 \to \infty$, the solution approaches $1 + \frac{m_2}{m_1 + m_2}$ optimal, with high probability. 
\end{corollary}

\begin{proof}
Because the total density of robot and obstacles are no more than $1/3$, if Theorem~\ref{t:minimax} extends to support regularly distributed obstacles, then Theorem~\ref{t:rth-ratio} applies because the highway heuristics do not pass through the obstacles. This is true because each obstacle can only add a constant length of path detour to an agent's path. In other words, the length $\ell$ in Theorem~\ref{t:rth-ratio} will only increase by a constant factor and will remain as $\ell = O(\log^{\frac{3}{4}}m)$. Similar arguments hold for Proposition~\ref{p:phase:1}.
\end{proof}

We now give the running time of \rtmapf and \rth. 

\begin{proposition}[Running Time, \rth]\label{p:time}
For $n \le \frac{m_1m_2}{3}$ agents on an $m_1 \times m_2$ grid, \rth runs in $O(nm_1^2m_2)$ time.
\end{proposition}

\begin{proof}
The running time of \rth is dominated by the matching computation and solving unlabeled \mpp. 
The matching part takes $O(m_1^2m_2)$ in deterministic time or $O(m_1m_2\log m_1)$ in expected time \cite{goel2013perfect}. 
Unlabeled \mpp may be tackled using the max-flow algorithm \cite{ford1956maximal} in $O(nm_1m_2T)=O(nm_1^2m_2)$ time, where $T=O(m_1+m_2)$ is the expansion time horizon of a time-expanded graph that allows a routing plan to complete. 
\end{proof}

\subsection{One-Half Agent Density: Shuffling with Linear Merging}\label{subsec:half-density}
Using a more sophisticated shuffle routine, $\frac{1}{2}$ agent density can be supported while retaining most of the guarantees for the $\frac{1}{3}$ density setting; obstacles are no longer supported. 

To best handle $\frac{1}{2}$ agent density, we employ a new shuffle routine called \emph{linear merge}, based on merge sort, and denote the resulting algorithm as Grid Rearrangement with Linear Merge or \rtlm. 
The basic idea is straightforward: for $m$ agents on a $2 \times m$ grid,  we iteratively sort the agents first on $2\times 2$ grids, then $2\times 4$ grids, and so on, much like how merge sort works. An illustration of the process on a $2 \times 8$ grid is shown in Fig.~\ref{fig:merge}.

\begin{algorithm}
\begin{small}
\DontPrintSemicolon
\SetKwProg{Fn}{Function}{:}{}
\SetKwFunction{LineMerge}{LineMerge}
\SetKwFunction{Merge}{Merge}

\caption{Line Merge Algorithm \label{alg:lineMerge}}

\KwIn{An array $arr$ representing the vertices labeled from $1$ to $n$ with current and intermediate locations of $n$ agents.}

\Fn{\LineMerge{$arr$}}{
    \If{$\text{length of } arr > 1$}{
        $mid \leftarrow \text{length of } arr \div 2$\;
        $left \leftarrow arr[0 \ldots mid-1]$\;
        $right \leftarrow arr[mid \ldots \text{length of } arr - 1]$\;
        
        \LineMerge{$left$}\;
        \LineMerge{$right$}\;
        
        \Merge{$arr$}\;
    }
}

\Fn{\Merge{$arr$}}{
      Sort agents located at the vertices in $arr$ to obtain intermediate states\;
    \For{$i\in arr$}{
        $a_i \leftarrow$ agent at vertex $i$\;
        \If{$a_i$.intermediate $> i$}{
            Route $a_i$ moving rightward using the bottom line while avoiding blocking those agents that are moving leftward\;
        }
        \Else{
            Route $a_i$ moving leftward using the upper line without stopping\;
        }
    }
    Synchronize the paths\;
}
\end{small}
\end{algorithm}

\begin{figure}[h!]
\centering
    \begin{overpic}[width=\linewidth]{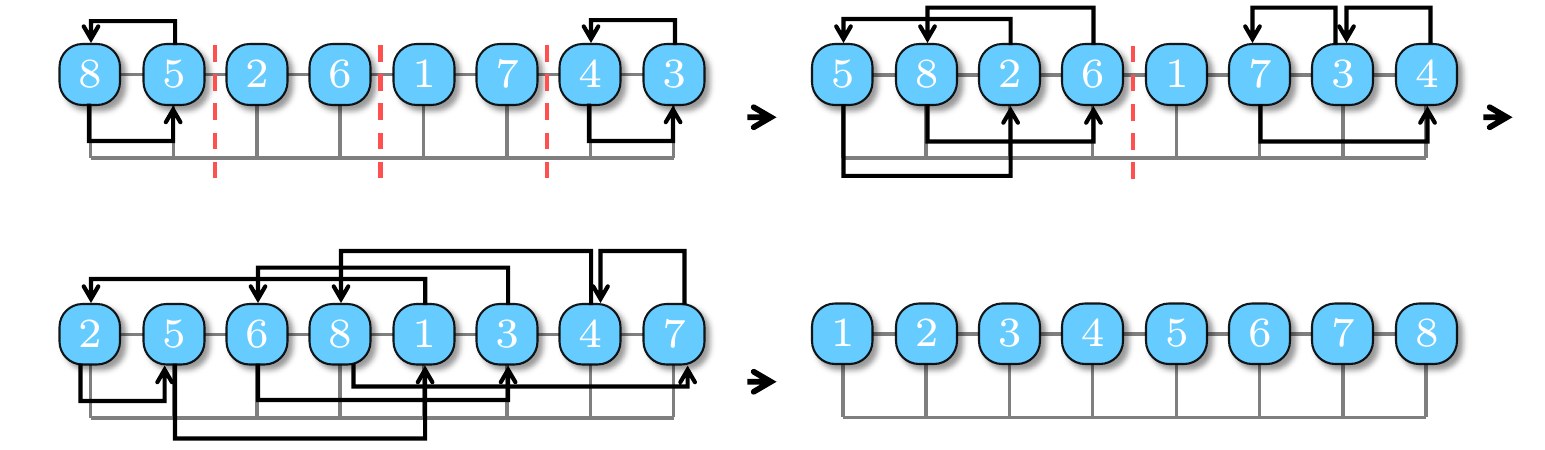}
             \footnotesize
             \put(22.8, 14.2) {(a)}
             \put(71.5, 14.2) {(b)}
             \put(22.8, -2.5) {(c)}
             \put(71.5, -2.5) {(d)}
    \end{overpic}
    \caption{A demonstration of the linear merge shuffle primitive on a $2\times 8$ grid. Agents going to the left always use the upper channel while agents going to the right always use the lower channel.} 
    \label{fig:merge}
\end{figure}
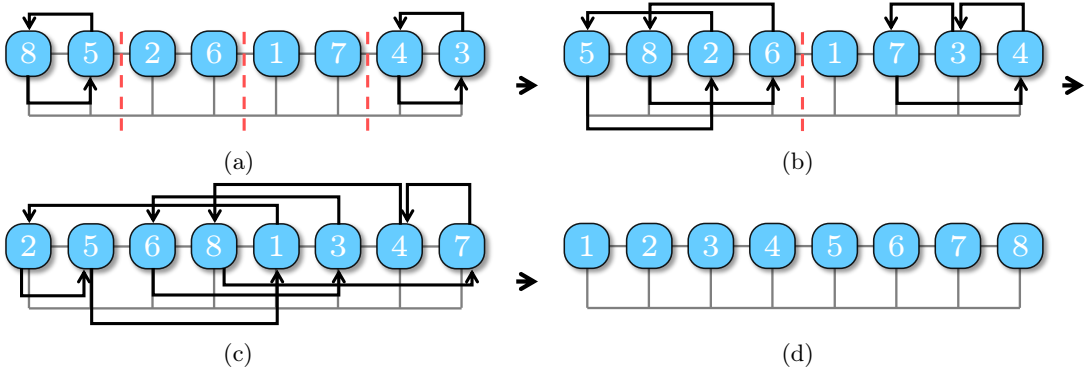


\begin{lemma}[Properties of Linear Merge]\label{l:lm}
On a $2\times m$ grid, $m$ agents, starting on the first row, can be arbitrarily ordered using $m + o(m)$ steps, using the Alg.~\ref{alg:lineMerge}, inspired from merge sort. The motion plan can be computed in polynomial time. 
\end{lemma}
\begin{proof}
We first show \emph{feasibility}. The procedure takes $\lceil \log m \rceil$ phases; in a phase, let us denote a section of the $2\times m$ grid where robots are treated together as a \emph{block}. For example, the left $2 \times 4$ grid in Fig.~\ref{fig:merge}(b) is a block. It is clear that the first phase, involving up to two robots per block, is feasible (i.e., no collision). Assuming that phase $k$ is feasible, we look at phase $k + 1$. We only need to show that the procedure is feasible on one block of length up to $2^{k+1}$. For such a block, the left half-block of length up to $2^k$ is already fully sorted as desired, e.g., in increasing order from left to right. For the $k+1$ phase, all robots in the left half-block may only stay in place or move to the right. These robots that stay must be all at the leftmost positions of the half-block and will not block the motions of any other robot. For the robots that need to move to the right, their relative orders do not need to change and, therefore, will not cause collisions among themselves. Because these robots that move in the left half-block will move down on the grid by one edge, they will not interfere with any robot from the right half-block. 
Because the same arguments hold for the right half-block (except the direction change), the overall process of merging a block occurs without collision. 

Next, we examine the \emph{makespan}. For any single robot $r$, at phase $k$, suppose it belongs to block $b$ and block $b$ is to be merged with block $b'$. It is clear that the robot cannot move more than $len(b') + 2$ steps, where $len(b')$ is the number of columns of $b'$ and the $2$ extra steps may be incurred because the robot needs to move down and then up the grid by one edge. This is because any move that $r$ needs to do is to allow robots from $b'$ to move toward $b$.
Because there are no collisions in any phase, adding up all the phases, no robot moves more than $m + 2(\log m + 1) = m + o(m)$ steps. 

Finally, the merge sort-like linear merge shuffle primitive runs in $O(m\log m)$ time since it is a standard divide-and-conquer routine with $\log m$ phases. 
\end{proof}
We distinguish between $\frac{1}{3}$ and $\frac{1}{2}$ density settings because the overhead in \rtlm is larger. Nevertheless, with the linear merge, the asymptotic properties of \rth for $\frac{1}{3}$ agent density mostly carry over to \rtlm. 

\begin{theorem}[Random \mpp, $\frac{1}{2}$ Agent Density]
For random \mpp instances on an $m_1 \times m_2$ grid, where $m_1 \ge m_2$ are multiples of two, for $\frac{m_1m_2}{3} \le n  \le \frac{m_1m_2}{2}$ agents, a solution can be computed in polynomial time that has makespan $m_1 + 2m_2 + o(m_1)$ with high probability.
As $m_1 \to \infty$, the solution approaches an optimality of $1 + \frac{m_2}{m_1 + m_2} \in (1, 1.5]$, with high probability. 
\end{theorem}
\tg{added proof sketch}
\begin{proof}[Proof Sketch]
The proof follows by combining Proposition~\ref{p:asymptotic_lowerbound} and Lemma~\ref{l:lm}.
\end{proof}

\subsection{Supporting Arbitrary \mpp Instances on Grids}\label{subsec:arbi}
We now examine applying \rth to arbitrary \mpp instances up to $\frac{1}{2}$ agent density.
If an \mpp instance is arbitrary, all that changes to \rth is the makespan it takes to complete the unlabeled reconfiguration phase. On an $m_1\times m_2$ grid, by computing a matching, it is straightforward to show that it takes no more than $m_1 + m_2$ steps to complete the unlabeled reconfiguration phase, starting from an arbitrary start configuration.  Since two executions of unlabeled reconfiguration are needed, this adds $2(m_1 + m_2)$ additional makespan. Therefore, the following results holds.

\begin{theorem}[Arbitrary \mpp, $\le \frac{1}{2}$ Density]\label{t:arbi}
For arbitrary \mpp instances on an $m_1 \times m_2$ grid, $m_1 \ge m_2$, for $n \le \frac{m_1m_2}{2}$ agents, a $3m_1 + 4m_2 + o(m_1)$ makespan solution can be computed in polynomial time. 
\end{theorem}
\tg{Obvious result, does it need a proof}
    

\section{Generalization to 3D}\label{sec: three_d}
We now explore the 3D setting. To keep the discussion focused, we mainly show how to generalize \rth to 3D, but note that \rtmapf and \rtlm can also be similarly generalized. 
High-dimensional \rtp \cite{szegedy2023rubik} is defined as follows. 

\begin{problem}[{\normalfont \bf Grid Rearrangement in $k$D (\rtpk)}]\label{p:rtkd}
Let $M$ be an $m_1 \times \ldots \times m_k$ table, $m_1\ge \ldots \ge m_k$, containing $\prod_{i=1}^km_i$ items, one in each table cell. 
In a \emph{shuffle} operation, the items in a single column in the $i$-th dimension of $M$, $1 \le i \le k$, may be permuted arbitrarily. 
Given two arbitrary configurations $S$ and $G$ of the items on $M$, find a sequence of shuffles that take $M$ from $S$ to $G$.
\end{problem}

We may solve \rtpthree using \rtatwo as a subroutine by treating \rtp as an \rtptwo, which is straightforward if we view a 2D slice of \rtatwo as a ``wide'' column. For example, for the $m_1 \times m_2 \times m_3$ grid, we may treat the second and third dimensions as a single dimension.
Then, each wide column, which is itself an $m_2 \times m_3$ 2D problem, can be reconfigured by applying \rtatwo.
With some proper counting, we obtain the following 3D version of 
Theorem~\ref{t:rta} as follows.
\begin{theorem}[{\normalfont \bf Grid Rearrangement Theorem, 3D}]\label{p:rta3d}
An arbitrary Grid Rearrangement problem on an $m_1\times m_2\times m_3$ table can be solved using $m_1m_2+m_3(2m_2+m_1)+m_1m_2$ shuffles. 
\end{theorem}

Although \cite{szegedy2023rubik} offers a broad framework for addressing \rtpk, it is not directly applicable to multi-agent routing in high-dimensional spaces. To overcome this limitation, we have developed the high-dimensional \mpp algorithm by incorporating and elaborating on its underlying principles similar to the 2D scenarios.
Denote the corresponding algorithm for Theorem~\ref{p:rta3d} as 
\rtathree, we illustrate how it works on a $3 \times 3\times 3$ table (Fig.~\ref{fig:rubik3d}). \rta operates in three phases. First, a bipartite graph $B(T, R)$ is constructed based on the initial table where the partite set $T$ are the colors of items representing the desired $(x,y)$ positions, and the set $R$ are the set of $(x,y)$ positions of items  (Fig.~\ref{fig:rubik}(b)). Edges are added as in the 2D case. 
From $B(T,R)$, $m_3$ \emph{perfect matchings} are computed. Each matching contains $m_1m_2$ edges and connects all of $T$ to all of $R$. processing these matchings yields an intermediate table (Fig.~\ref{fig:rubik3d}(c)), serving a similar function as in the 2D case.
After the first phase of $m_1m_2$ $z$-shuffles, the intermediate table (Fig.~\ref{fig:rubik3d}(c)) that can be reconfigured by applying \rtatwo with $m_1$ $x$-shuffles and $2m_2$ $y$-shuffles.
This sorts each item in the correct $x$-$y$ positions (Fig.~\ref{fig:rubik}(d)).
Another $m_1m_2$  $z$-shuffles can then sort each item to its desired $z$ position (Fig.~\ref{fig:rubik}(e)). 

 \begin{figure}[!htbp]
        \centering
        \includegraphics[width=\linewidth]{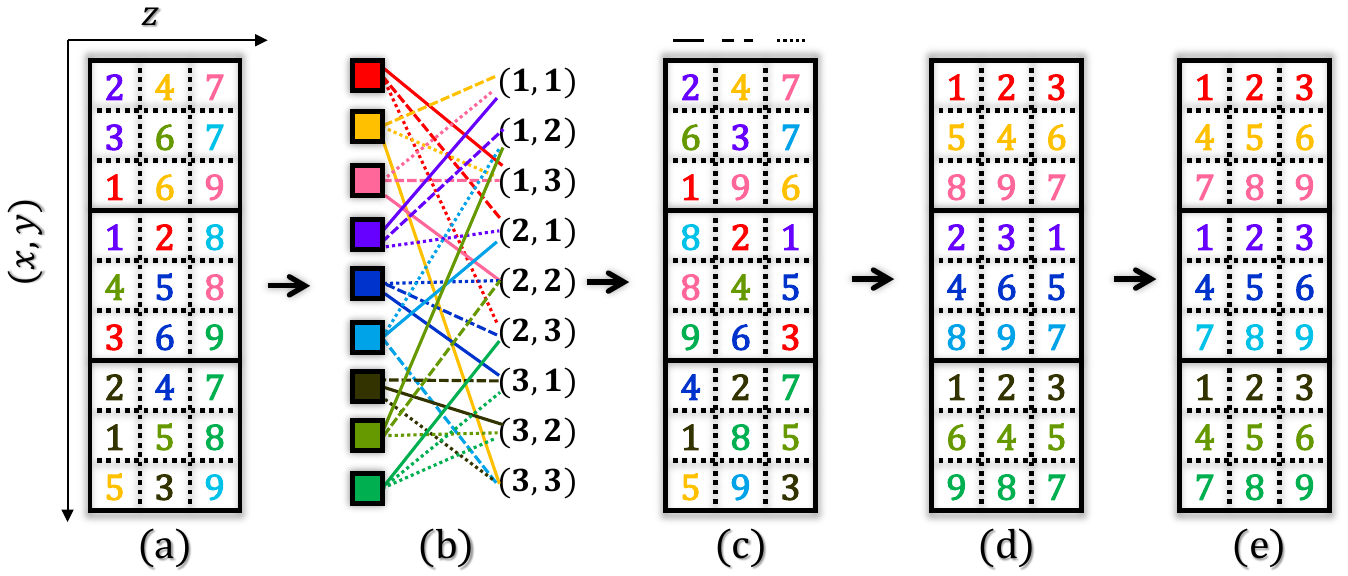}
                \caption{Illustration of applying \emph{\rtathree}. (a) The initial $3\times 3\times 3$ table with  a random arrangement of 27 items that are colored and labeled. Color represents the $(x,y)$ position of an item. (b) The constructed bipartite graph. The right partite set contains all the possible $(x,y)$ positions. It contains $3$ perfect matchings, determining the $3$ columns in (c). (c) Applying $z$-shuffles to (a), according to the matching results, leads to an intermediate table where each $x$-$y$ plane has one color appearing exactly once. (d) Applying wide shuffles to (c) correctly places the items according to their $(x, y)$ values (or colors). (e)  Additional $z$-shuffles fully sort the labeled items.} 
        \label{fig:rubik3d}
    \end{figure}
\subsection{Extending \rth to 3D}
Alg.~\ref{alg:rth3d}, which calls Alg.~\ref{alg:matching3d} and Alg.~\ref{alg:xyfit}, outlines the high-level process of extending \rth to 3D. 
In each $x$-$y$ plane, $\mathcal G$ is partitioned into $3\times 3$ cells (e.g.,Fig.~\ref{fig:example}). 
Without loss of generality, we assume that $m_1,m_2,m_3$ are multiples of 3 and there are no obstacles.
First, to make \rtathree applicable, we convert the arbitrary start/goal
configurations to intermediate \emph{balanced} configurations $S_1$ and $G_1$, treating agents as unlabeled, as we have done in \rth, wherein each 2D plane, each $3 \times 3$ cell contains no more than $3$ agents (Fig.~\ref{fig:random_to_balanced}). 

 \begin{figure}[!htbp]
        \centering
        \includegraphics[width=0.9\linewidth]{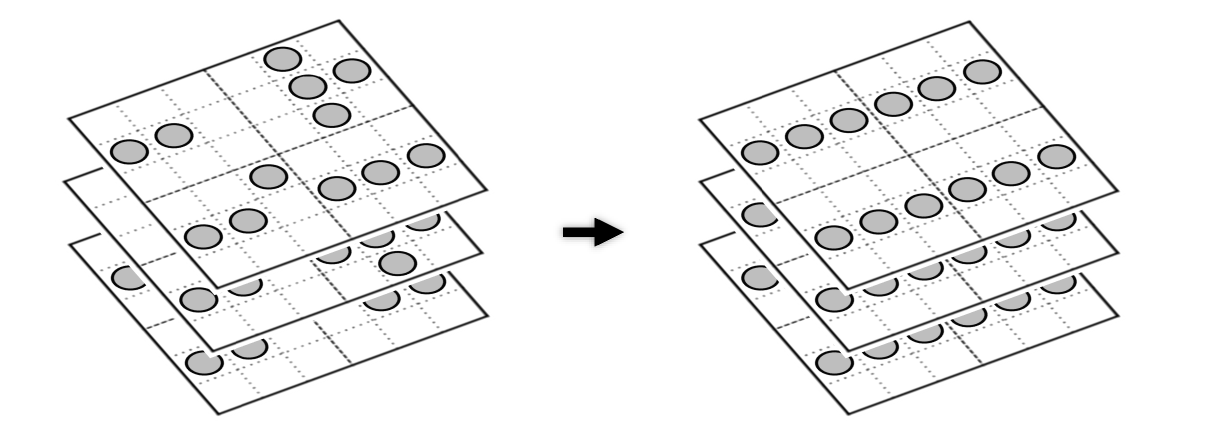}
\caption{Applying unlabeled \mpp to convert a random configuration to a balanced centering one on  $6\times 6 \times 3$ grids.} 
        \label{fig:random_to_balanced}
    \end{figure}

\rtathree can then be applied to coordinate the agents moving toward their intermediate goal positions $G_1$.
Function $\texttt{MatchingXY}$ finds a feasible intermediate configuration $S_2$ and routes the agents to $S_2$ by simulating shuffle operations along the $z$ axis.
Function $\texttt{XY-Fitting}$ apply shuffle operations along the $x$ and $y$ axes to route each agent $i$ to its desired $x$-$y$ position  $(g_{1i}.x,g_{1i}.y)$.
In the end, the function $\texttt{Z-Fitting}$ is called, routing each agent $i$ to the desired $g_{1i}$ by performing shuffle operations along the $z$ axis and concatenating the paths computed by unlabeled \mpp planner $\texttt{UnlabeledMRPP}$. 

\begin{algorithm}
\begin{small}
\DontPrintSemicolon
\SetKwProg{Fn}{Function}{:}{}
\SetKwFunction{Fampp}{UnlabeledMRPP}
\SetKwFunction{Frthddd}{GRH3D}
\SetKwFunction{Fmatchingthreed}{MatchingXY}
\SetKwFunction{Frthdd}{XY-Fitting}
\SetKwFunction{Fzfitting}{Z-Fitting}

  \caption{\rthddd \label{alg:rth3d}}
  \KwIn{Start and goal vertices $S=\{s_i\}$ and $G=\{g_i\}$}
  \Fn{\Frthddd{$S,G$}}{
$S_1,G_1\leftarrow$\Fampp{$S,G$}\;
\Fmatchingthreed{}\;
\Frthdd{}\;
\Fzfitting{}\;
}
\vspace{1mm}
\end{small}
\end{algorithm}  
We now explain each part of \rthddd.
$\texttt{MatchingXY}$ uses an extended version of \rta to find perfect matching that allows feasible shuffle operations. 
Here, the ``item color" of an item $i$ (agent) is the tuple $(g_{1i}.x,g_{1i}.y)$, which is the desired $x$-$y$ position it needs to go.
After finding the $m_3$ perfect matchings, the intermediate configuration $S_2$ is determined.
Then, shuffle operations along the $z$ direction can be applied to move the agents to $S_2$.
\begin{algorithm}
\begin{small}
\DontPrintSemicolon
\SetKwProg{Fn}{Function}{:}{}
\SetKwFunction{Fmatchingthreed}{MatchingXY}
  \caption{MatchingXY \label{alg:matching3d}}
  \KwIn{Balanced start and goal vertices $S_1=\{s_{1i}\}$ and $G_1=\{g_{1i}\}$}
  \Fn{\Fmatchingthreed{$S_1,G_1$}}{
$A\leftarrow[1,...,n]$\;
$\mathcal{T}\leftarrow$ the set of $(x,y)$ positions in $S_1$\;
\For{$(r,t)\in \mathcal{T}\times \mathcal{T}$}{
    \If{$\exists i\in A$ where $(s_{1i}.x,s_{1i}.y)=r\wedge (g_{1i}.x,g_{1i}.y)=t$}{
        add edge $(r,t)$ to $B(T,R)$\;
        remove $i$ from $A$ \;
    }

}
    compute matchings $\mathcal{M}_1,...,\mathcal{M}_{m_3}$ of $B(T, R)$\;
     $A\leftarrow [1,...,n]$\;
 \ForEach{$\mathcal{M}_c$ and $(r,t)\in \mathcal{M}_c$}{
 \If{$\exists i\in A$ where $(s_{1i}.x,s_{1i}.y)=r\wedge (g_{1i}.x,g_{1i}.y)=t$}{
 $s_{2i}\leftarrow (s_{1i}.x, s_{1i}.y,c)$ and remove $i$ from $A$\;
  mark agent $i$ to go to $s_{2i}$\;
 }
 }
   perform simulated $z$-shuffles in parallel \;
  
}
\vspace{1mm}
\end{small}
\end{algorithm}
The agents in each $x$-$y$ plane will be reconfigured by applying $x$-shuffles and $y$-shuffles.
We need to apply \rtatwo for these agents in each plane, as demonstrated in Alg.~\ref{alg:xyfit}.
In \rth, for each 2D plane,  the ``item color" for agent $i$ is its desired $x$ position $g_{1i}.x$.
For each plane, we compute the $m_2$ perfect matchings to determine the intermediate position $g_{2}$.
Then each agent $i$ moves to its $g_{2i}$ by applying $y$-shuffle operations.
In Line 18, each agent is routed to its desired $x$ position by performing $x$-shuffle operations.
In Line 19, each agent is routed to its desired $y$ position by performing $y$-shuffle operations.
\begin{algorithm}
\begin{small}
\DontPrintSemicolon
\SetKwProg{Fn}{Function}{:}{}
\SetKwFunction{Fxyfitting}{XY-Fitting}
\SetKwFunction{Frthdd}{GRH2D}
  \caption{XY-Fitting \label{alg:xyfit}}
  \KwIn{Current positions $S_2$ and balanced goal positions $G_1$}
  \Fn{\Fxyfitting{}}{
  \For{$z\leftarrow[1,...,m_3]$}{
$A\leftarrow \{i|s_{2i}.y=z\}$\;
\Frthdd{$A,z$}\;
}
}
  \Fn{\Frthdd{$A,z$}}{

$\mathcal{T} \leftarrow$ the set of $x$ positions of $S_2$\;
\For{$(r,t)\in \mathcal{T}\times\mathcal{T}$}{
       \If{$\exists i\in A$ where $s_{2i}.x=r\wedge g_{1i}.x=t$}{
       \If{agent $i$ is not assigned}{
       add edge $(r,t)$ to $B(T,R)$\;
        mark $i$ assigned \;
       }

    }
      compute matchings $\mathcal{M}_1,...,\mathcal{M}_{m_2}$ of $B(T, R)$\;
}
$A'\leftarrow A$\;
   \ForEach{$\mathcal{M}_c$ and $(r,t)\in \mathcal{M}_c$}{
 \If{$\exists i\in A'$ where $s_{2i}.x=c\wedge g_{1i}.x=t$}{
 $g_{2i}\leftarrow (s_{2i}.x, c,z)$ and remove $i$ from $A'$\;
  mark agent $i$ to go to $g_{2i}$\;
 }
 }
  route each agent $i\in A$ to $g_{2i}$\;
  route each agent $i\in A$ to  $(g_{2i}.x,g_{1i}.y,z)$\;
  route each agent $i\in A$ to  $(g_{1i}.x,g_{1i}.y,z)$\;
}

\vspace{1mm}
\end{small}
\end{algorithm} 
After all the agents reach the desired $x$-$y$ positions, another round of $z$-shuffle operations in $\texttt{Z-Fitting}$ can route the agents to the balanced goal configuration computed by an unlabeled \mpp planner.
In the end, we concatenate all the paths as the result.

\rtlm and \rtmapf can be extended to 3D in similar ways by replacing the solvers in 2D planes.
For \rtmapf, there is no need to use unlabeled \mpp for balanced reconfiguration, which yields a makespan upper bound of $4m_1 + 8m_2 + 8m_3$ (as a direct combination of Theorem~\ref{t:rtm-makespan} and Theorem~\ref{p:rta3d}). For arbitrary instances under half density in 3D, the makespan guarantee in Theorem~\ref{t:arbi} updates to  $3m_1 + 4m_2 + 4m_3 + o(m_1)$.

\subsection{Properties of \rthddd}
Computation for \rthddd is dominated by perfect matching and unlabeled \mpp.
Finding $m_3$ ``wide column'' matchings takes $O(m_3 m_1^2m_2^2)$ deterministic time or $O(m_3m_1m_2\log(m_1m_2))$ expected time.
We apply \rtatwo for simulating ``wide column'' shuffle, which requires $O(m_3 m_1^2m_2)$  deterministic time or $O(m_1m_2m_3\log m_1)$ expected time. 
Therefore, the total time complexity for the Grid Rearrangement part is $O(m_3 m_1^2m_2^2+m_1^2m_2m_3)$.
For unlabeled \mpp, we can use the max-flow based algorithm \cite{yu2013multi} to compute the makespan-optimal solutions. 
The max-flow portion can be solved in $O(n|E|T)=O(n^2(m_1+m_2+m_3))$  where $|E|$ is the number of edges and $T=O(m_1+m_2+m_3)$ is the time horizon of the time expanded graph \cite{ford1956maximal}.
%
%
One can also perform two ``wide column'' shuffles plus one $z$ shuffle, which yields $2(2m_1+m_2)+m_3$ number of shuffles and $O(m_1m_2m_3^2+m_3m_1m_2^2)$. This requires more shuffles but a shorter running time.

Next, we derive the optimality guarantee, for which the upper and lower bounds on the makespan are needed. These are straightforward generalizations of results in 2D.

\begin{theorem}[Makespan Upper Bound]\label{p:expected_makespan}
\rthddd returns solutions with $m_1+2m_2+2m_3+o(m_1)$ asymptotic makespan for \mpp instances with $\frac{m_1m_2m_3}{3}$ random start and goal configurations on 3D grids, with high probability.
\rtlmddd returns solutions with $m_1+2m_2+2m_3+o(m_1)$ asymptotic makespan for \mpp instances with $\frac{m_1m_2m_3}{2}$ random start and goal configurations on 3D grids, with high probability.
\rtmddd returns solutions with $3m_1+6m_2+6m_3+o(m_1)$ asymptotic makespan for \mpp instances with $m_1m_2m_3$ random start and goal configurations on 3D grids, with high probability.
Moreover, if $m_1=m_2=m_3=m$, \rthddd and \rtlmddd returns solutions with  $5m+o(m)$ makespan, \rtmddd returns solutions with $15m+o(m)$ makespan.
\end{theorem}

\begin{proposition}[Makespan Lower Bound]\label{p:asymptotic_lowerbound}
For \mpp instances on $ m_1\times  m_2\times m_3$ grids with $\Theta(m_1m_2m_3)$ random start and goal configurations on 3D grids, the makespan lower bound is asymptotically approaching $m_1+m_2+m_3$, with high probability.
\end{proposition}



\begin{theorem}[Makespan Optimality Ratio]\label{c:fh}
\rthddd and \rtlmddd yield asymptotic $1+\frac{m_2+m_3}{m_1+m_2+m_3}$ makespan optimality ratio for \mpp instances with $\Theta(m_1m_2m_3)\le \frac{m_1m_2m_3}{3}$ and $\le \frac{m_1m_2m_3}{2}$ random start and goal configurations respectively on 3D grids, with high probability.
\rtmddd yields asymptotic $3+\frac{3(m_1+m_2)}{m_1+m_2+m_3}$ makespan optimality ratio for \mpp instances with $\Theta(m_1m_2m_3)\le m_1m_2m_3$ random start and goal configurations on 3D grids with high probability.
\end{theorem} 



We may further generalize the result to higher dimensions. 

\begin{theorem}
 Consider a $k$-dimensional cubic grid with grid size $m$. 
 For uniformly distributed start and goal configurations,  \rth and \rtlm can solve the instances with asymptotic makespan optimality being $\frac{2^{k-1}+1}{k}$ and \rtmapf yields $\frac{4(2^{k-1}+1)}{k}$ asymptotic makespan optimality.
\end{theorem}

\begin{proof}
Unlabeled \mpp takes $o(m)$ makespan (note for $k\geq 3$, the minimax grid matching distance is $O(\log ^{1/k}m)$ ~\cite{shor1991minimax}).
Extending proposition \ref{p:asymptotic_lowerbound} to $k$-dimensional grid, the asymptotic lower bound is $mk$.
We prove that the asymptotic makespan $f(k)$ is $(2^{k-1}+1)m+o(m)$ by induction.
The Grid Rearrangement algorithm solves a $k$-dimensional problem by using two 1-dimensional shuffles and one $(k-1)$-dimensional ``wide column" shuffle.
Therefore, we have $f(k)=2m+f(k-1)$.
It's trivial to see $f(1)=m+o(m),f(2)=3m+o(m)$, which yields that $f(k)=2^{k-1}m+m+o(m)$ and makespan optimality ratio being $\frac{2^{k-1}+1}{k}$ for \rth and \rtlm while $\frac{4(2^{k-1}+1)}{k}$ for \rtmapf.
\end{proof}

\section{Optimality-Boosting Heuristics}\label{sec:opt-boost}
\subsection{Reducing Makespan via Optimizing Matching}\label{subsec:heuristics}
Based on \rta, \rth has three simulated shuffle phases. The makespan is dominated by the agent needing the longest time to move, as a sum of moves in all three phases. As a result, the optimality of Grid Rearrangement  methods is determined by the first preparation/matching phase. 
Finding arbitrary perfect matchings is fast but the process can be improved to reduce the overall makespan. 

For improving matching, we propose two heuristics; the first is based on \emph{integer programming} (IP). 
We create binary variables $\{x_{ri}\}$ where $r$ represents the row number and $i$ the agent. 
agent $i$ is assigned to row $r$ if $x_{ri}=1$. 
Define single agent cost as $C_{ri}(\lambda)=\lambda |r-s_i.x|+(1-\lambda)|r-g_i.x|$. 
We optimize the makespan lower bound of the first phase by letting $\lambda=0$ or the third phase by letting $\lambda=1$. 
The objective function and constraints are given by
\begin{equation}
\label{eq:objective}
    \max_{r,i} \{{C_{ri}(\lambda=0)x_{ri}}\}+\max_{r,i}\{C_{ri}(\lambda=1)x_{ri}\}
\end{equation}
\vspace{-2mm}
\begin{equation}
\label{eq:constraint1}
    \sum_{r}x_{ri}=1, \mathrm{for\ each\ agent\ } i
\end{equation}
\vspace{-2mm}
\begin{equation}
\label{eq:constraint2}
    \sum_{g_i.y=t}x_{ri}\leq 1, {\small\mathrm{for\ each\ row\ }r\mathrm{\ and\ each\ color\ } t}  \end{equation}
\vspace{-1mm}
\begin{equation}
\label{eq:constraint3}
\sum_{s_i.x=c}x_{ri}=1, \mathrm{for\ each\ column\ }c \mathrm{and\ each\ row\ }r
\end{equation}

Eq. \eqref{eq:objective} is the summation of makespan lower bound of the first phase and the third phase. Note that the second phase can not be improved by optimizing the matching.
Eq. \eqref{eq:constraint1} requires that agent $i$ be only present in one row.
Eq. \eqref{eq:constraint2} specifies that each row should contain agents that have different goal columns.
Eq. \eqref{eq:constraint3} specifies that each vertex $(r,c)$ can only be assigned to one agent. 
The IP model represents a general assignment problem which is NP-hard in general.
It has limited scalability but provides a way to evaluate how optimal the matching could be in the limit.

A second matching heuristic we developed is based on \emph{linear bottleneck assignment (LBA)} \cite{burkard2012assignment}, which takes polynomial time.
LBA differs from the IP heuristic in that the bipartite graph is weighted.
For the matching assigned to row $r$, the edge weight of the bipartite graph is computed greedily.
If column $c$ contains agents of color $t$, we add an edge $(c,t)$ and its edge cost is
\vspace{-1.5mm}
\begin{equation}
\vspace{-1.5mm}
    C_{ct}=\min_{g_i.y=t}C_{ri}(\lambda=0)
\end{equation}
We choose $\lambda=0$ to optimize the first phase. Optimizing the third phase ($\lambda=1$) would give similar results.
After constructing the weighted bipartite graph, an $O(\frac{m_1^{2.5}}{\log m_1})$ LBA algorithm \cite{burkard2012assignment} is applied to get a minimum bottleneck cost matching for row $r$. Then we remove the assigned agents and compute the next minimum bottleneck cost matching for the next row. 
After getting all the matchings $\mathcal{M}_r$, we can further use LBA to assign $\mathcal{M}_r$ to a different row $r'$ to get a smaller makespan lower bound. The cost for assigning matching $\mathcal{M}_r$ to row $r'$ is defined as 
\vspace{-1.5mm}
\begin{equation}
\vspace{-1.5mm}
    C_{\mathcal{M}_rr'}=\max_{i\in\mathcal{M}_r}C_{r'i}(\lambda=0)
\end{equation}
The total time complexity of using LBA heuristic for matching is $O(\frac{m_1^{3.5}}{\log m_1})$.

We denote \rth with IP and LBA heuristics as \rthip and \rthlba, respectively.
We mention that \rtmapf, which uses the line swap motion  primitive, can also benefit from these heuristics to re-assign the goals within each group. This can lower the bottleneck path length and improve optimality. 

\subsection{Path Refinement}\label{sec: path_refine}
Final paths from \rta-based algorithms are concatenations of paths from multiple planning phases. 
This means agents are forced to ``synchronize'' at phase boundaries, which causes unnecessary idling for agents finishing a phase early. 
Noticing this, we de-synchronize the planning phases, which yields significant gains in makespan optimality. 

Our path refinement method does something similar to Minimal Communication Policy (\mcp)~\cite{ma2017multi}, a multi-agent execution policy proposed to handle unexpected delays without stopping unaffected agents.
During  execution, \mcp let agents execute their next non-idling move as early as possible while preserving the relative execution orders between agents,
e.g., 
when two agents need to enter the same vertex at different times, that ordering is preserved. 
%
%
%
%

We adopt the principle used in \mcp  to refine the paths generated by \rta-based algorithms as shown in Alg.~\ref{alg:refinement}, with Alg.~\ref{alg:mcp_move} as a sub-routine.  
The algorithms work as follows. All idle states are removed from the initial agent but the order of visits for each vertex is kept (Line 2-3).
Then we enter a loop executing the plans following the \mcp principle (Line 8-10).
In Alg.~\ref{alg:mcp_move}, if  $i$ is the next agent that enters vertex $v_i$ according to the original order, we check if there is a agent currently at $v_i$.
If there is not, we let $i$ enter $v_i$.
If another agent $j$ is currently occupying $v_j$, we examine if $j$ is moving to its next vertex $v_j$ in the next step by recursively calling the function 
$\texttt{Move}$. 
We check if there is any cycle in the agent movements in the next original plan.
If $i$ is in a cycle, we move all the agents in this cycle to their next vertex recursively (Line 20-28).
If no cycle is found and $j$ is to enter its next vertex $v_j$, we let $i$ also move to its next vertex $v_i$. 
Otherwise, $i$ should  wait at $u_i$. 
The algorithm is \emph{deadlock-free} by construction ~\cite{ma2017multi}.

\begin{algorithm}
\begin{small}
\DontPrintSemicolon
\SetKwProg{Fn}{Function}{:}{}
\SetKw{Continue}{continue}

 \caption{Path Refinement \label{alg:refinement}}
  \KwIn{Paths $\mathcal{P}$ generated by \rta }
  \Fn{Refine($\mathcal{P}$)}{
  \textbf{foreach} $v\in V$, $VOrder[v]\leftarrow Queue()$\;
  \texttt{Preprocess($InitialPlans,VOrder$)}\;
      \While{true}{
          \For{$i=1...n$}{
            $Moved\leftarrow Dict()$\;
            $CycleCheck\leftarrow Set()$\;
            $\texttt{Move}(i)$\;
              \If{$\texttt{AllReachedGoals()}$=true}{break\;}
        }
    }
}
\end{small}
\end{algorithm}

\begin{algorithm}
\begin{small}
\DontPrintSemicolon
\SetKwProg{Fn}{Function}{:}{}
\SetKw{Continue}{continue}
\caption{Move \label{alg:mcp_move}}
\Fn{$\texttt{Move}$(i)}{
\If{$i$ in $Moved$}{
\Return $Moved[i]$\;
}
$u_i\leftarrow$ current position of $i$\;
$v_i\leftarrow$ next position of $i$\;
\If{$i=VOrder[v_i].front()$}{
    $j\leftarrow$ the agent currently at $v_i$\;
    \If{$i$ is in $CycleCheck$}{
        $\texttt{MoveAllAgentsInCycle}(i)$\;
        \Return true\;
    }
    $CycleCheck.add(i)$\;
    \If{$\texttt{Move}(j)=true$ or $j=$None}{
       let $i$ enter $v_i$\;
       $VOrder[v_i].popfront()$\;
    $Moved[i]\leftarrow$true\;
    \Return true\;
    }
}
let $i$ wait at $u_i$\;
$Moved[i]$=false\;
\Return false\;
}
\Fn{$\texttt{MoveAllAgentsInCycle}(i)$}{
$Visited\leftarrow Set()$\;
$j\leftarrow i$\;
\While{$j$ is not in visited}{
    let $j$ enter its next vertex $v_j$\;
    $Visited.add(j)$\;
    $VOrder[v_j].popfront()$\;
    $Moved[j]\leftarrow$ true\;
    $j\leftarrow$ the agent currently at vertex $v_j$\; 
}
}

\vspace{-2mm}
\end{small}
\end{algorithm}
%
Let $T$ be the makespan of the initial paths, the makespan of the solution obtained by running $\texttt{Refine}$ is clearly no more than $T$.
In each loop of $\texttt{Refine}$, we essentially run DFS on a graph that has $n$ nodes and traverse all the nodes, for which the time complexity is $O(n)$,
Therefore the time complexity of the path refinement is bounded by $O(nT)$.

Other path refinement methods, such as \cite{li2021anytime,okumura2021iterative}, can also be applied in principle, which iteratively chooses a small group of agents and re-plan their paths holding other agents as dynamic obstacles. 
However, in dense settings that we tackle, re-planning for a small group of agents has little chance of finding better paths this way. 

\section{Simulation Experiments}\label{sec:eval}
We thoroughly evaluated \rta-based algorithms and compared them with many similar algorithms. 
We mainly highlight comparisons with \eecbs($w$=1.5) \cite{li2021eecbs}, \lacam~\cite{okumura2023lacam} and \ddm \cite{han2020ddm}. 
These two methods are, to our knowledge, two of the fastest near-optimal \mpp solvers.
Beyond \eecbs and \ddm, we considered a state-of-the-art polynomial algorithm, push-and-swap \cite{luna2011push}, which gave fairly sub-optimal results: the makespan optimality ratio is often above 100 for the densities we examine. 
%

As a reader's guide to this section, in Sec.~\ref{e:1}, as a warm-up, we show the 2D performance of all \rta-based algorithms at their baseline, i.e., without any efficiency-boosting heuristics mentioned in Sec.~\ref{sec:opt-boost}. 
In Sec.~\ref{e:2}, for 2D square grids, we show the performance of all \rta-based algorithms with and without the two heuristics discussed in Sec.~\ref{sec:opt-boost}. We then thoroughly evaluate the performance of all versions of the \rth2d algorithm at $\frac{1}{3}$ agent density in Sec.~\ref{subsec:rth}. Some special 2D patterns are examined in Sec.~\ref{e:4}. 3D settings are briefly discussed in Sec.~\ref{e:5}. 

All experiments are performed on an Intel\textsuperscript{\textregistered} Core\textsuperscript{TM} i7-6900K CPU at 3.2GHz. Each data point is an average of over 20 runs on randomly generated instances unless otherwise stated.
A running time limit of $300$ seconds is imposed over all instances. 
The optimality ratio is estimated as compared to conservatively estimated makespan lower bounds.
All the algorithms are implemented in C++.
%
We choose Gurobi \cite{gurobi} as the mixed integer programming solver and  ORtools \cite{ortools} as the max-flow solver.

\subsection{Optimality of Baseline Versions of \rta-Based Methods}\label{e:1}

First, we provide an overall evaluation of the optimality achieved by basic versions of \rtmapf, \rtlm, and \rth over randomly generated 2D instances at their maximum designed agent density. 
That is, these methods do not contain the heuristics from Sec.~\ref{sec:opt-boost}.
We test over three $m_1:m_2$ ratios: $1:1$, $3:2$, and $5:1$. 
For \rtmapf, different sub-grid sizes for dividing the $m_1 \times m_2$ grid are evaluated. 
%
The result is plotted in Fig.~\ref{fig:RTM-RTLM-RTH}. 
Computation time is not listed; we provide the computation time later for \rth; the running times of \rtmapf, \rtlm, and \rth are all similar. 
The optimality ratio is computed as the ratio between the solution makespan and the longest Manhattan distance between any pair of start and goal, which is conservative. 
\begin{figure}[htbp]
        \centering
        \includegraphics[width=1\linewidth]{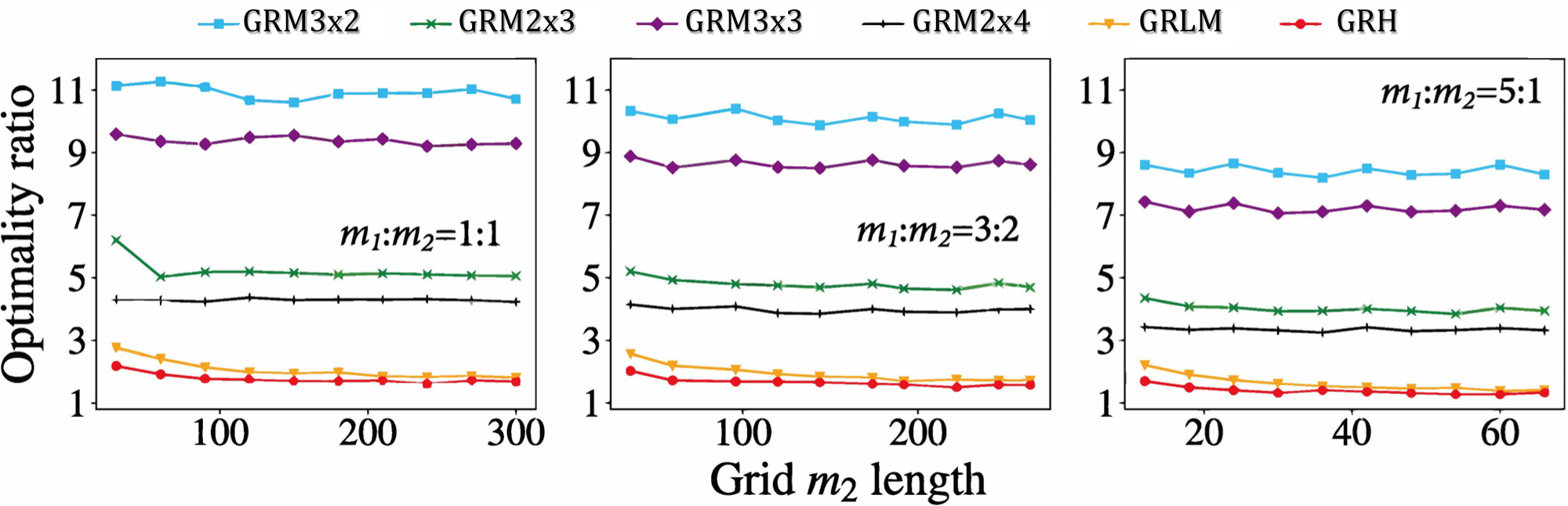}
\caption{Makespan optimality ratio for  \rtmapf (3x2, 2x3, 3x3, 2x4), \rtlm, and \rth at their maximum design density, for different $m_2$ values and $m_1:m_2$ ratios. The largest \rtmapf problems have $90,000$ agents on a $300 \times 300$ grid.} 
        \label{fig:RTM-RTLM-RTH}
    \end{figure}

\rtmapf does better and better on the optimality ratio as the sub-grid size ranges from $3\times 2$, $3\times 3$, $2\times 3$, and $2\times 4$, dropping to just above $3$.
%
In general, using ``longe'' sub-grids for the shuffle will decrease the optimality ratio because there are opportunities to reduce the overhead.
However, the time required for computing the solutions for all possible configurations grows exponentially as the size of the sub-grids increases.
%

On the other hand, both \rtlm and \rth achieve a sub-2 optimality ratio in most test cases, with the result for \rth dropping below $1.5$ on large grids. For all settings, as the grid size increases, there is a general trend of improvement of optimality across all methods/grid aspect ratios. This is due to two reasons: (1) the overhead in the shuffle operations becomes relatively smaller as grid size increases, and (2) with more agents, the makespan lower bound becomes closer to $m_1 + m_2$. Lastly, as $m_1:m_2$ ratio increases, the optimality ratio improves as predicted. For many test cases, the optimality ratio for \rth at $m_1:m_2=5$  is around $1.3$.

\begin{table}[h]
\small
  \centering
  \begin{tabular}{|c|c|c|c|c|c|}
    \hline
    \textbf{$\#$ of Agents} & 5 & 10  &15  &20 &25 \\
\hline
\textbf{Optimality Gap} & 1 &1.0025 &1.004 &1.011  &1.073\\
\hline
  \end{tabular}
  \caption{Optimality gap investigation on $5\times 5$ grids }
  \label{tab:OptimalityGap}
\end{table}

The exploration of optimality gaps is conducted on $5\times 5$ grids, as shown in Table.~\ref{tab:OptimalityGap}. For every specified number of agents, we create 100 random instances and employ the ILP solver to solve them. The optimality gap is then assessed by calculating the average ratio between the optimal makespan and the makespan lower bound. 
The optimality gap widens with higher agent density.

\subsection{Evaluation and Comparative Study of \rth}\label{subsec:rth}
\subsubsection{Impact of grid size}
For our first detailed comparative study of the performance of \rta,\rtlm and \rth at $100\%$, $\frac{1}{2}$ and $\frac{1}{3}$ density respectively, we set $m_1:m_2 = 3:2$ in terms of computation time and makespan optimality ratio.
We compare with \ddm \cite{han2020ddm}, EECBS ($w=1.5$) \cite{li2021eecbs}, Push and Swap\cite{luna2011push}, \lacam \cite{okumura2023lacam} in Fig.~\ref{fig:revise_full}-\ref{fig:revise_third}. For \eecbs, we turn on all the available heuristics and reasonings.
%



        \begin{figure}[htbp]
        \centering
        \includegraphics[width=1\linewidth]{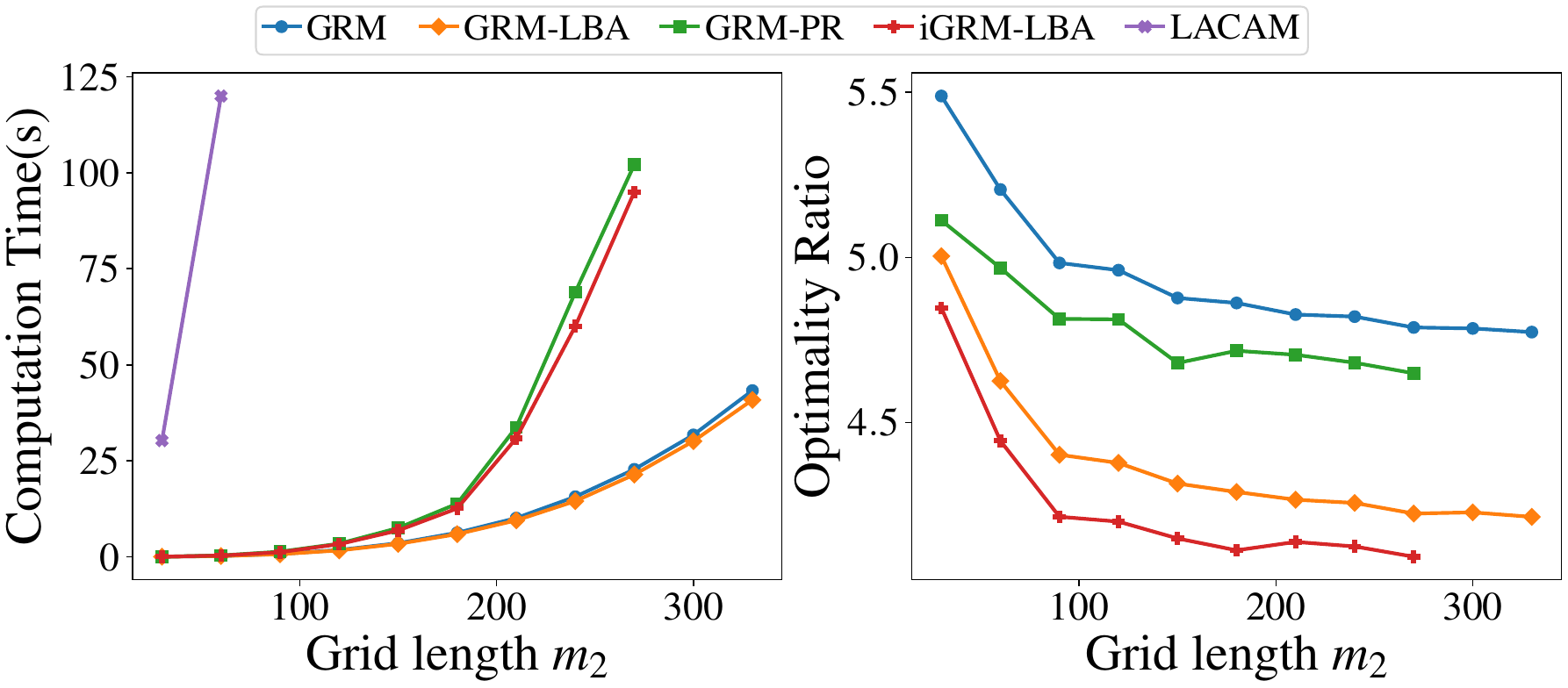}
        \caption{Computation time and optimality ratios on $m_1 \times m_2$ grids of varying sizes with $m_1:m_2 = 3:2$ and agent density at $100\%$ density.} 
        \label{fig:revise_full}
    \end{figure}

            \begin{figure}[htbp]
        \centering
        \includegraphics[width=1\linewidth]{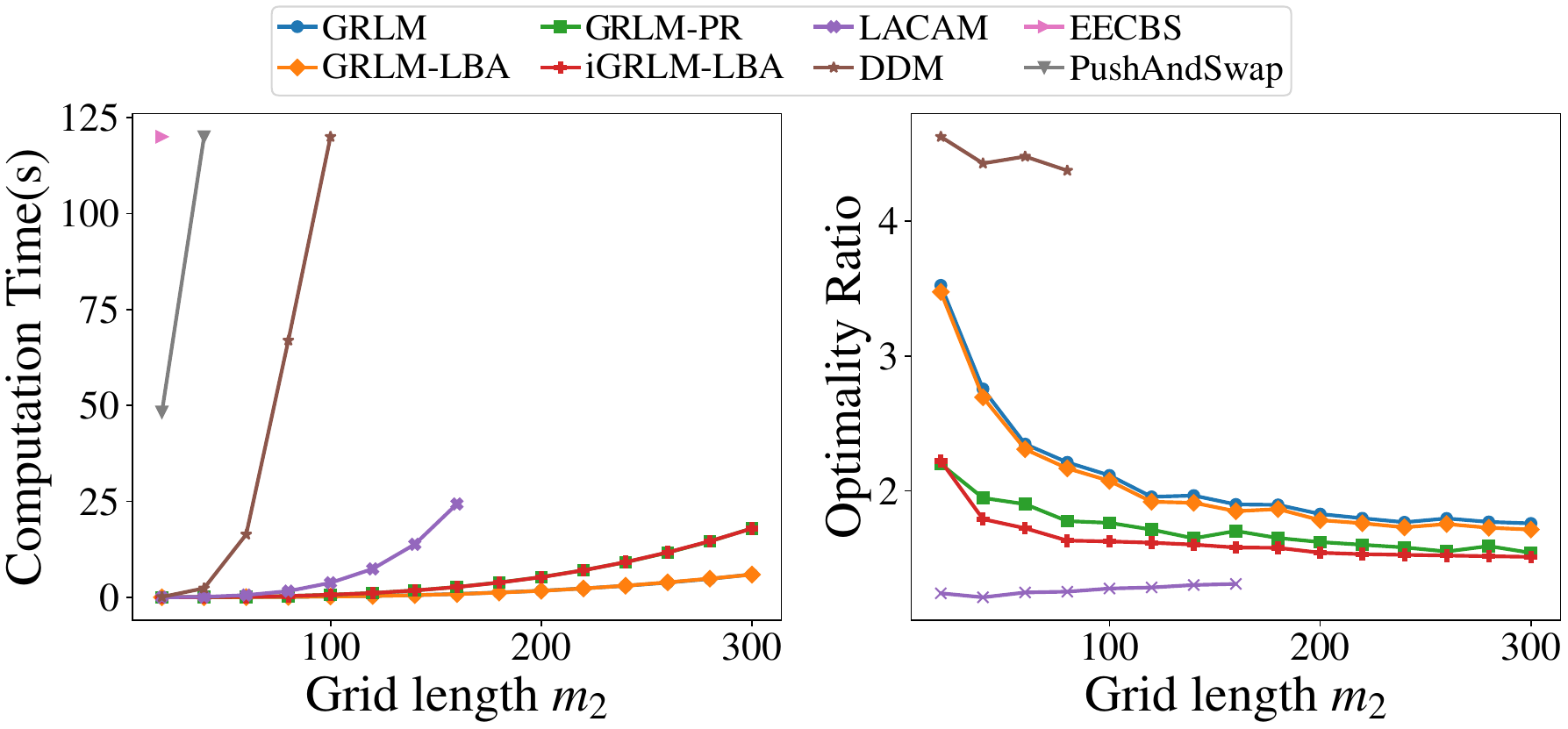}
        \caption{Computation time and optimality ratios on $m_1 \times m_2$ grids of varying sizes with $m_1:m_2 = 3:2$ and agent density at $\frac{1}{2}$ density.} 
        \label{fig:revise_half}
    \end{figure}

        \begin{figure}[htbp]
        \centering
        \includegraphics[width=1\linewidth]{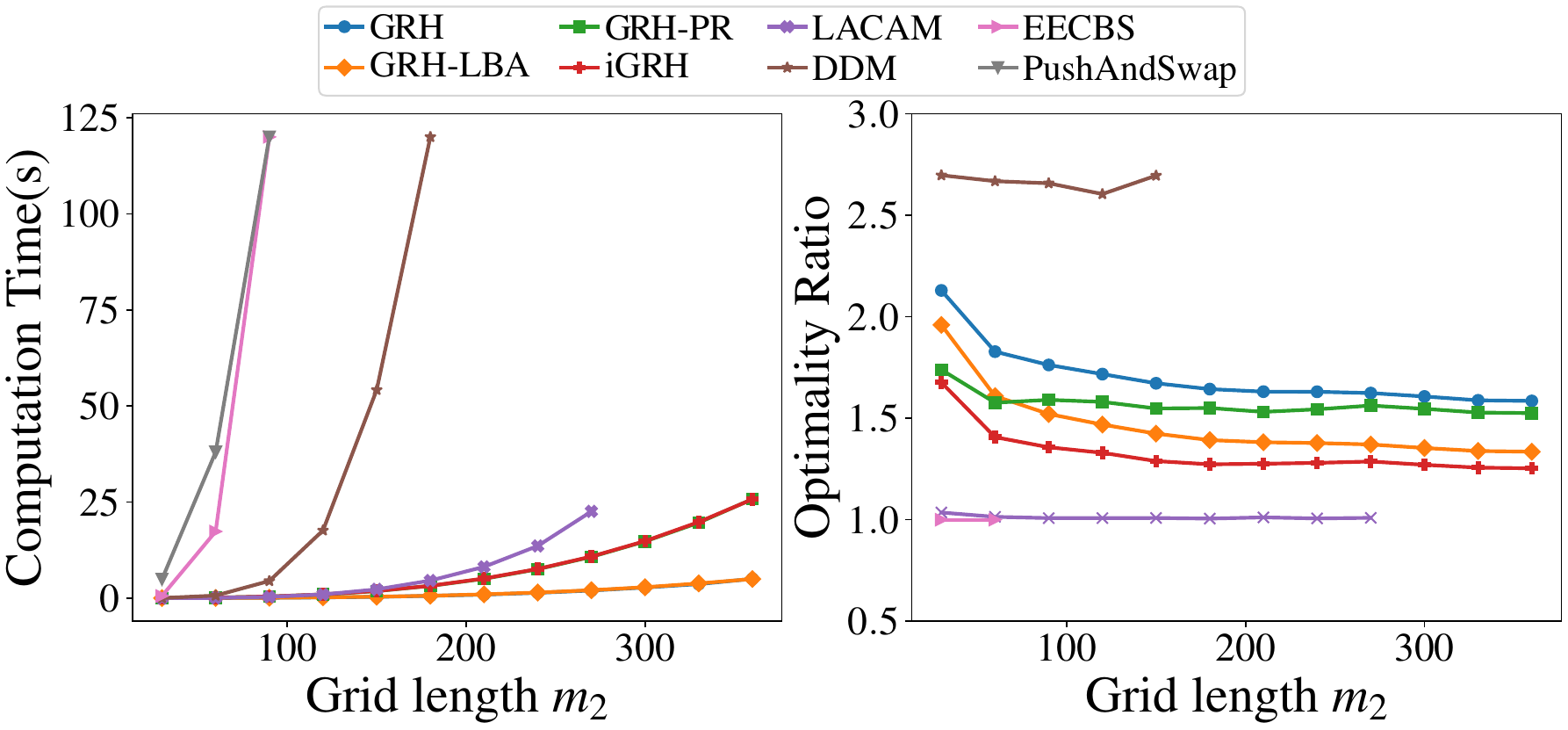}
        \caption{Computation time and optimality ratios on $m_1 \times m_2$ grids of varying sizes with $m_1:m_2 = 3:2$ and agent density at $1/3$ density.} 
        \label{fig:revise_third}
    \end{figure}
    
%
When at $100\%$ agent density, \rtmapf methods can solve huge instances, e.g., on $450 \times 300$ grids with $135,000$ agents in about $40$ seconds while none of the other algorithms can.
\lacam can only solve instances when $m_2=30$, though the resulting makespan optimality is around $2073$ and thus is not shown in the figure.
When at $\frac{1}{2}, \frac{1}{3}$ agent density, \rtlm and \rth still are the fastest methods among all, scaling to 45,000 agents in 10 seconds while the optimality ratio is close to 1.5 and 1.3 respectively.
\lacam also achieves great scalability and is able to solve problems when $m_2 \leq 270$. However, after that, \lacam faces out-of-memory error. This is due to the fact that \lacam is a search algorithm on joint configurations, and the required memory grows exponentially as the number of agents. In contrast, \rth, \rtlm do not have the issue and solve the problems consistently despite the optimality being worse than \lacam.
\eecbs, stopped working after $m_2 = 90$ at $\frac{1}{3}$ agent density and cannot solve any instance within the time limit at $100\%$ and $\frac{1}{2}$ agent density, while \ddm stopped working after $m_2=180$.
Push and Swap also performed poorly, and its optimality ratio is more than 30 in those scenarios and thus is not presented in the figure.

\subsubsection{Handling obstacles}
\rth can also handle scattered obstacles and is especially suitable for cases where obstacles are regularly distributed. 
For instance, problems with underlying graphs like that in Fig. \ref{fig:jd_center}(b), where each $3 \times 3$ cell has a hole in the middle, can be natively solved without performance degradation.
Such settings find real-world applications in parcel sorting facilities in large warehouses \cite{wan2018lifelong,li2020lifelong}.
For this parcel sorting setup, we fix the agent density at $\frac{2}{9}$ and test \eecbs, \ddm, \rth, \rthlba, \rthpr, i\rth on graphs with varying sizes. 
The results are shown in Fig \ref{fig: sorting_random}.
Note that \ddm can only apply when there is no narrow passage. So we added additional ``borders'' to the map to make it solvable for \ddm.
The results are similar as before; we note that i\rth reaches a conservative optimality ratio estimate of $1.26$.
\tg{for the sortation map, add push and swap to keep consistency}
   \begin{figure}[htbp]
        \centering
        \includegraphics[width=1\linewidth]{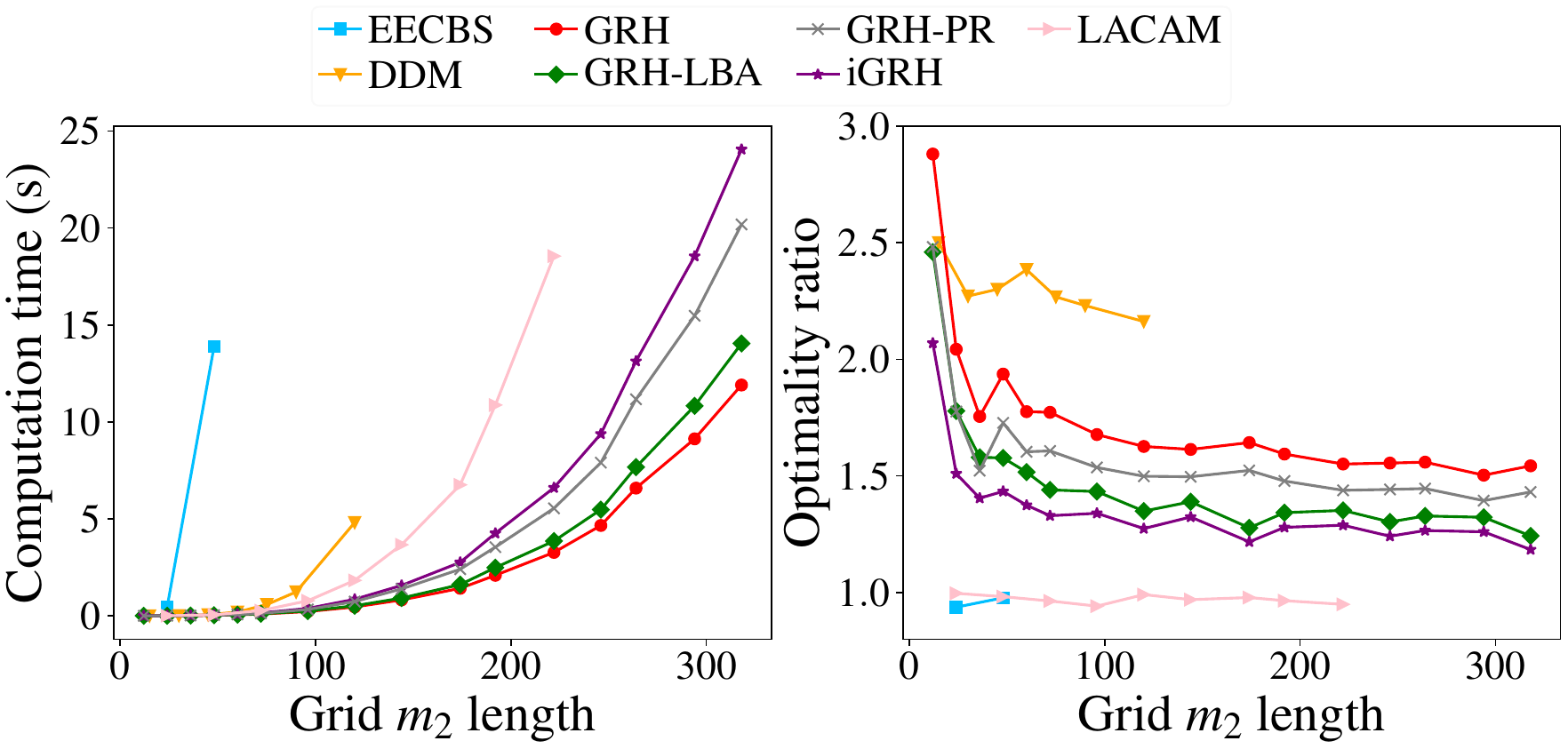}
        \caption{Computation time and optimality ratios on environments of varying sizes with regularly distributed obstacles at $\frac{1}{9}$ density and agents at $\frac{2}{9}$ density. $m_1:m_2 = 3:2$.
        } 
        \label{fig: sorting_random}
    \end{figure}
\tg{ also need to replace this figure}
\subsubsection{Impact of grid aspect ratios}
In this section, we fix $m_1m_2=90000$ and vary the $m_2:m_1$ ratio between $0$ (nearly one dimensional) and $1$ (square grids). We evaluated four algorithms, two of which are \rth and i\rth. Now recall that \rtp on an $m_1 \times m_2$ table can also be solved using $2m_2$ column shuffles and $m_1$ row shuffles. Adapting \rth and i\rth with $m_1 + 2m_2$ shuffles gives the other two variants we denote as \rth-LL and i\rth-LL, with ``LL'' suggesting two sets of longer shuffles are performed (each set of column shuffle work with columns of length $m_1$). The result is summarized in Fig.~\ref{fig:rectangle}.
  \begin{figure}[htbp]

        \centering
        \includegraphics[width=1\linewidth]{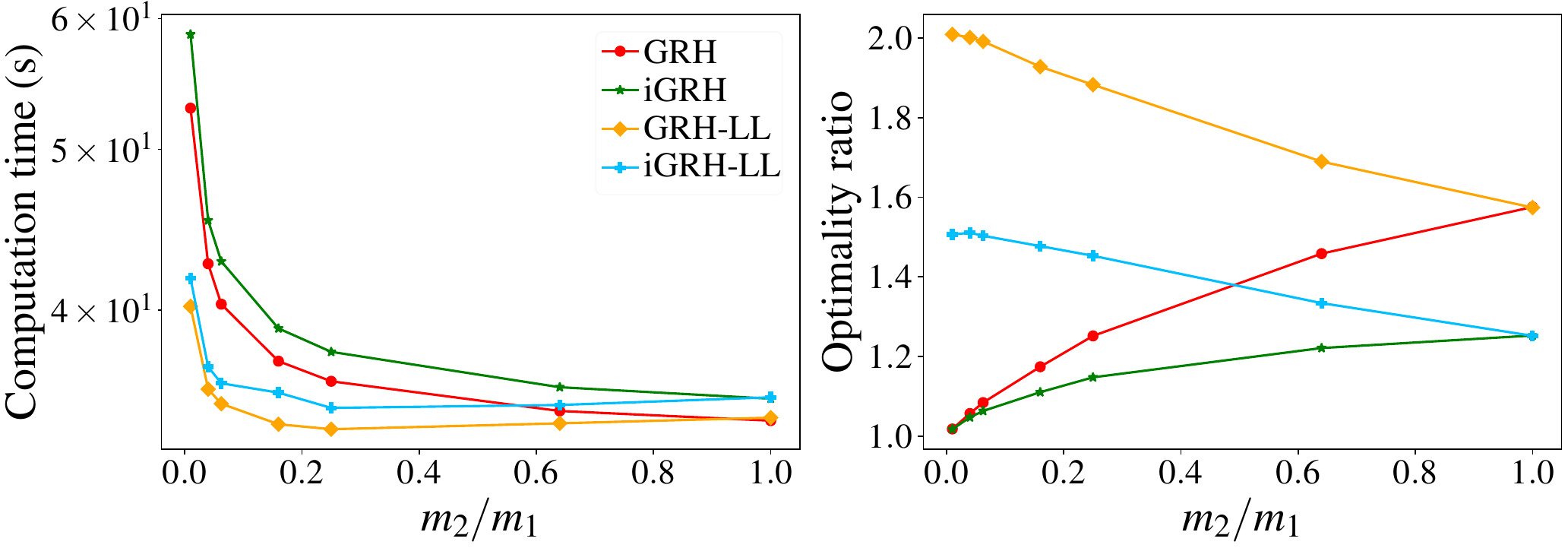}
        \caption{Computation time and optimality ratios on rectangular grids of varying aspect ratio and $\frac{1}{3}$ agent density.} 
        \label{fig:rectangle} 
    \end{figure}
    
Interestingly but not surprisingly, the result clearly demonstrates the trade-offs between computation effort and solution optimality. \rth and i\rth achieve better optimality ratio in comparison to \rth-LL and i\rth-LL but require more computation time.
Notably, the optimality ratio for \rth and i\rth is very close to 1 when $m_2:m_1$ is close to 0.
As expected, i\rth does much better than \rth across the board. 
\subsection{Special Patterns}\label{e:4}
We experimented i\rth on many ``special'' instances, two are presented here (Fig.~\ref{fig:new_test}). For both settings, we set $m_1 = m_2$. In the first, the ``squares'' setting, agents form concentric square rings and each agent and its goal are centrosymmetric. 
In the second, the ``blocks'' setting, the grid is divided into smaller square blocks (not necessarily $3 \times 3$) containing the same number of agents. Agents from one block need to move to another randomly chosen block.
i\rth achieves optimality that is fairly close to 1.0 in the square setting and 1.7 in the block setting. The computation time is similar to that of Fig.~\ref{fig: sorting_random}; EECBS performs well in terms of optimality, but its scalability is limited, working only on grids with $m \leq 90$. On the other hand, \lacam exhibits excellent scalability and good optimality for block patterns, although its optimality is comparatively worse for square patterns. Other algorithms are excluded from consideration due to significantly poorer optimality; for example, DDM's optimality exceeds 9, and Push\&Swap's optimality is greater than 40.

\tg{ change it to EECBS}
    \begin{figure}[htbp]
        \centering
           \begin{overpic}               
        [width=1\linewidth]{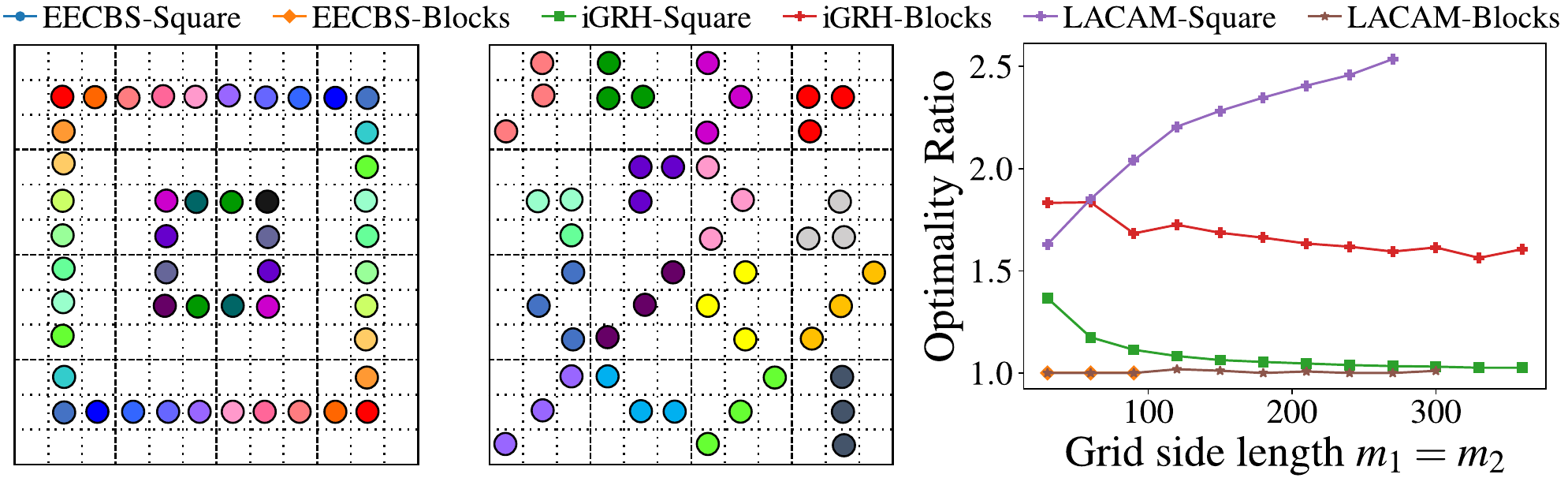}
             \footnotesize
             \put(12.5, -1) {(a)}
             \put(43.5, -1) {(b)}
                          \put(81.5, -1) {(c)}
        \end{overpic}
        \caption{(a) An illustration of the ``squares'' setting. (b) An illustration of the ``blocks'' setting. (c) Optimality ratios for the two settings for \eecbs, \lacam, and \rthlba.} 
        \label{fig:new_test} 
    \end{figure}

\subsection{Effectiveness of Matching and Path Refinement Heuristics}\label{e:2}
In this subsection, we evaluate the effectiveness of heuristics introduced in Sec.~\ref{sec:opt-boost} in boosting the performance of baseline \rta-based methods (we will briefly look at the IP heuristic later). We present the performance of \rtmapf, \rtlm, and \rth at these methods' maximum design density. For each method, results on all $4$ combinations with the heuristics are included.  
For a baseline method X, X-LBA, X-PR, and iX mean the method with the LBA heuristic, the path refinement heuristic, and both heuristics, respectively. 
In addition to the makespan optimality ratio, we also evaluated \emph{sum-of-cost}  (SOC) optimality ratio, which may be of interest to some readers. The sum-of-cost is the sum of the number of steps taking individual agents to reach their respective goals. 
We tested the worst-case scenario, i.e., $m = m_1 = m_2$. 
The result is shown in Fig. ~\ref{fig:path-refinement}. 
\begin{figure}[h!]
    \centering
    \includegraphics[width=1\linewidth]{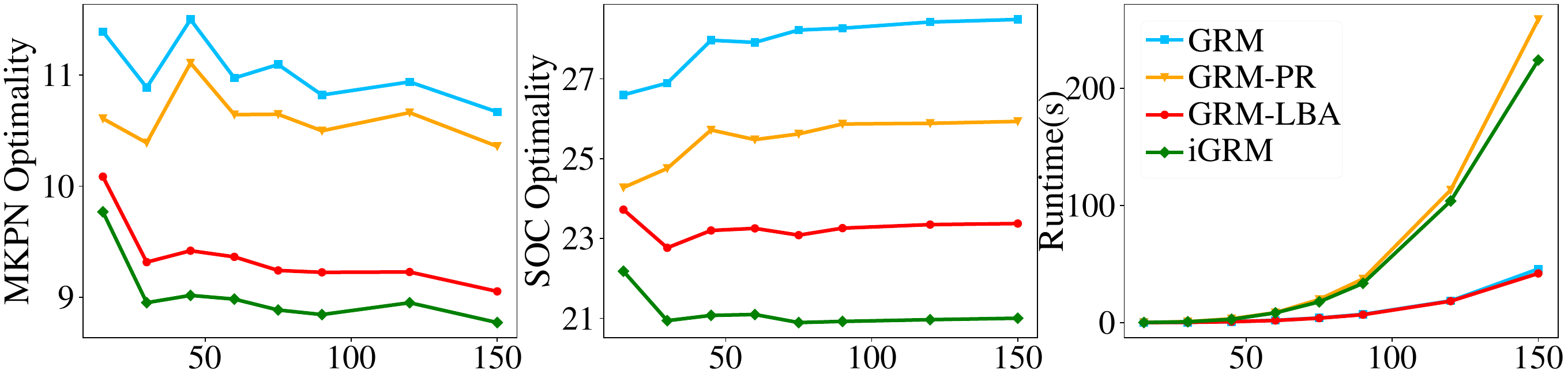}
    \includegraphics[width=1\linewidth]{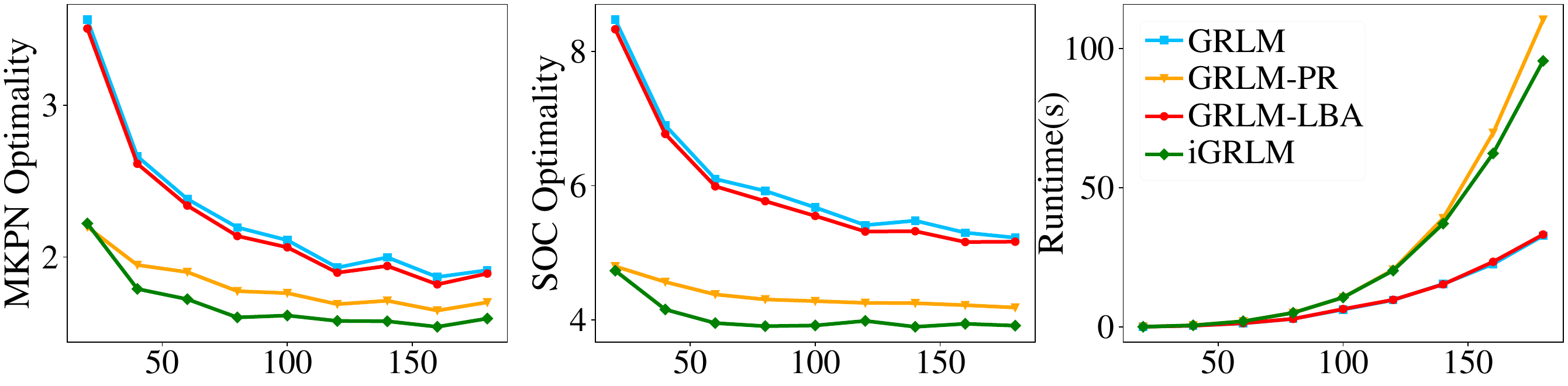}
    \includegraphics[width=1\linewidth]{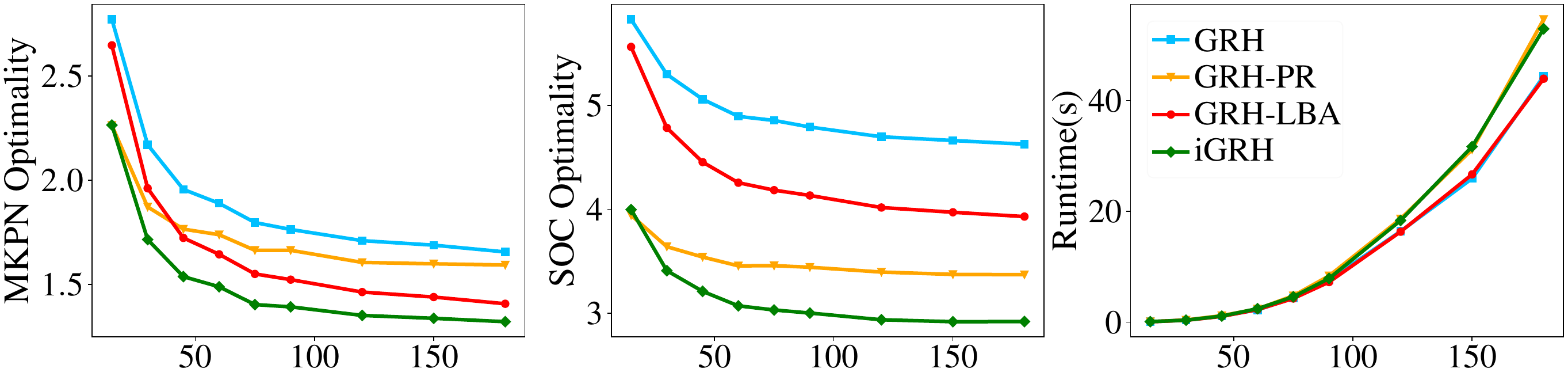}
    \caption{Effectiveness of heuristics in boosting \rta-based algorithms on $m\times m$ grids. For all figures, the $x$-axis is the grid side length $m$. Each row shows a specific algorithm (\rtmapf, \rtlm, and \rth). From left to right, makespan(MKPN) optimality ratio, SOC optimality ratio, and the computation time required for the path refinement routine are given.} 
    \label{fig:path-refinement} 
\end{figure}

We make two key observations based on Fig.~\ref{fig:path-refinement}. First, both heuristics provide significant individual boosts to nearly all baseline methods (except \rtlm-LBA), delivering around $10\%$--$20\%$ improvement on makespan optimality and $30\%$--$40\%$ improvement on SOC optimality. Second, the combined effect of the two heuristics is nearly additive, confirming that the two heuristics are orthogonal two each other, as their designs indicate. 
%
%
The end result is a dramatic overall cross-the-board optimality improvement. As an example, for \rth, for the last data point, the makespan optimality ratio dropped from around $1.8$ for the based to around $1.3$ for i\rth. 
In terms of computational costs, the LBA heuristic adds negligible more time. The path refinement heuristic takes more time in full and $\frac{1}{2}$ density settings but adds little cost in the $\frac{1}{3}$ density setting. 


\subsection{Evaluations on 3D Grids}\label{e:5}
For the 3D setting, the performance and solution structure of our methods are largely similar to the 2D setting. As such, we provide basic evaluations for completeness, fixing the aspect ratio at $m_1:m_2:m_3=4:2:1$ and density at $\frac{1}{3}$, and examine \rthddd variants on obstacle-free grids with varying sizes.
Here, because DDM only applies to 2D, we use ILP with split heuristics \cite{yu2016optimal} instead of DDM.
Start and goal configurations are randomly generated; the results are shown in Fig.~\ref{fig:random_rec3d}.
ILP with 16-split heuristic and \eecbs compute solution with better optimality ratio but does not scale.
In contrast, \rthddd variants readily scale to grids with over $370,000$ vertices and $120,000$ agents.
Optimality ratios for \rth variants decrease as the grid size increases, approaching 1.5-1.7.

\begin{figure}[htbp]
        \centering
        \includegraphics[width=1\linewidth]{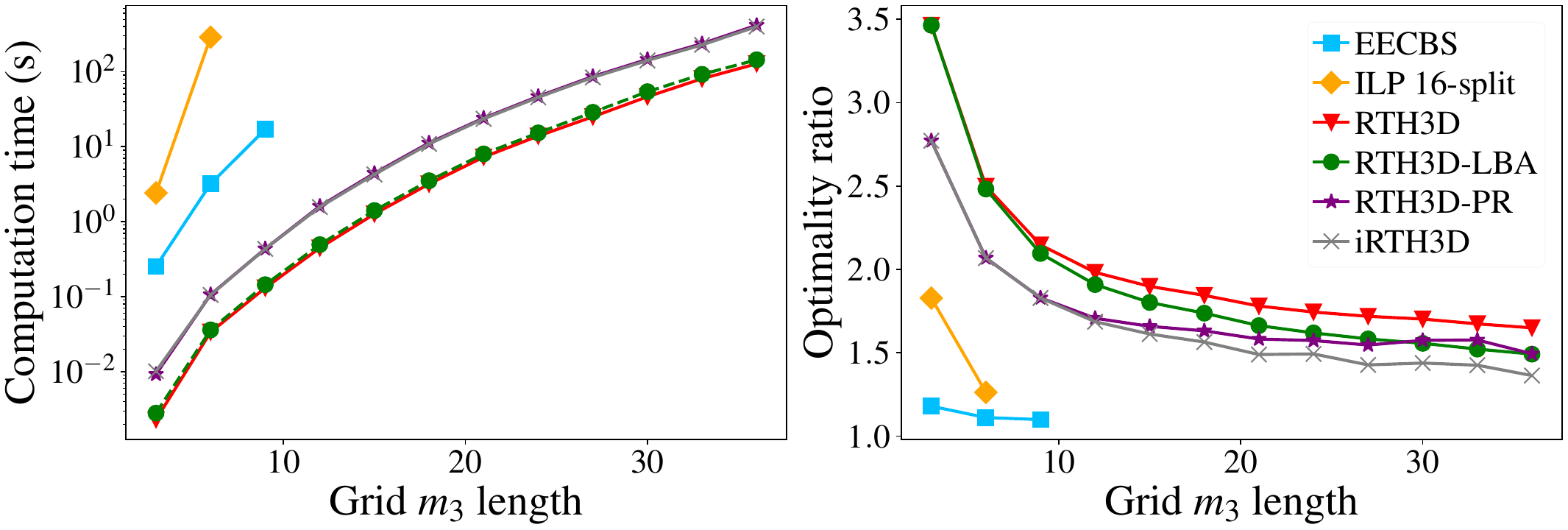}
\caption{Computation time and optimality ratio for four methods on $m_1\times m_2 \times m_3$ grids with varying grid size and $m_1:m_2:m_3=4:2:1$.} 
        \label{fig:random_rec3d}
    \end{figure}

\section{Conclusion and Discussion}\label{sec:conclusion}
In this study, we propose to apply Grid Rearrangements \cite{szegedy2023rubik} to solve \mpp. 
A basic adaptation of \rta, with a more efficient line shuffle routine, enables solving \mpp on grids at maximum agent density, in polynomial time, with a previously unachievable optimality guarantee. 
Then, combining \rta, a highway heuristic, and additional matching heuristics, we obtain novel polynomial time algorithms that are provably asymptotically $1 + \frac{m_2}{m_1+m_2}$ makespan-optimal on $m_1\times m_2$ grids with up to $\frac{1}{3}$ agent density, with high probability. 
Similar guarantees are also achieved with the presence of obstacles and at an agent density of up to one-half. 
These results in 2D are then shown to readily generalize to 3D and higher dimensions. 
In practice, our methods can solve problems on 2D graphs with over $10^5$ number of vertices and $4.5 \times 10^4$ agents to $1.26$ makespan-optimal (which can be better with a larger $m_1:m_2$ ratio). Scalability is even better in 3D.
To our knowledge, no previous \mpp solvers provide dual guarantees on low-polynomial running time and practical optimality. 

\tg{added discussion of limitation}
\textbf{Limitation}
While the \rta excels in achieving reasonably good optimality within polynomial time for dense instances, it does have limitations when applied to scenarios with a small number of robots or instances that are inherently easy to solve. 
In such cases, although \rta remains functional, its performance may lag behind other algorithms, and its solutions might be suboptimal compared to more specialized or efficient approaches. 
The algorithm's strength lies in its ability to efficiently handle complex, densely populated scenarios, leveraging its unique methodology. However, users should be mindful of its relative performance in simpler instances where alternative algorithms may offer superior solutions. 
This limitation underscores the importance of considering the specific characteristics of the robotic system and task at hand when choosing an optimization algorithm, tailoring the selection to match the intricacies of the given problem instance.

Our study opens the door for many follow-up research directions; we discuss a few here. 

\textbf{New line shuffle routines}. Currently, 
\rtlm and \rth only use two/three rows to perform a simulated row shuffle.
Among other restrictions, this requires that the sub-grids used for performing simulated shuffle be well-connected (i.e. obstacle-free or the obstacles are regularly spaced so that there are at least two rows that are not blocked by static obstacles in each motion primitive to simulate the shuffle).
Using more rows or irregular rows in a simulated row shuffle, it is possible to accommodate larger obstacles and/or support density higher than one-half.

\textbf{Better optimality at lower agent density}. 
It is interesting to examine whether further optimality gains can be realized at lower agent density settings, e.g., $\frac{1}{9}$ density or even 
lower, which are still highly practical. We hypothesize that this can be realized by somehow merging the different phases of \rta so that some unnecessary agent travel can be eliminated after computing an initial plan. 

\textbf{Consideration of more realistic robot models}.
The current study assumes a unit-cost model in which an agent takes a unit amount of time to travel a unit distance and allows turning at every integer time step. 
In practice, robots will need to accelerate/decelerate and also need to make turns. 
Turning can be especially problematic and cause a significant increase in plan execution time if the original plan is computed using the unit-cost model mentioned above. 
We note that \rth returns solutions where robots move in straight lines most of the time, which is advantageous compared to all existing \mpp algorithms, such as ECBS and DDM, which have many directional changes in their computed plans. 
it would be interesting to see whether the performance of \rta-based \mpp algorithms will further improve as more realistic robot models are adapted. 

\begin{remark}
    Currently, our \rta-based \mpp solves are limited to a static setting whereas e-commerce applications of multi-agent motion planning often require solving life-long settings  \cite{Ma2017LifelongMP}. 
%
The metric for evaluating life-long \mpp is often the \emph{throughput}, namely the number of goals reached per time step.
We note that \rth also provides optimality guarantees for such settings, e.g., for the setting where $m_1 = m_2 = m$,
the direct application of \rth to large-scale life-long \mpp on square grids yields an optimality ratio of $\frac{2}{9}$ on throughput.

We may solve life-long \mpp using \rth in \emph{batches}. 
For each batch with $n$ agents, \rth takes about $3m$ steps; the throughput 
is then $\mathcal{T}_{RTH} = \frac{n}{3m}$.
As for the lower bound estimation of the throughput, the expected Manhattan 
distance in an $m\times m$ square, ignoring inter-agent collisions, is 
$\frac{2m}{3}$.
Therefore, the lower bound throughput for each batch is $\mathcal{T}_{lb}=\frac{3n}{2m}$.
The asymptotic optimality ratio is $\frac{\mathcal{T}_{RTH}}{\mathcal{T}_{lb}}=\frac{2}{9}$.
The $\frac{2}{9}$ estimate is fairly conservative because \rth supports much higher 
agent densities not supported by known life-long \mpp solvers. 
Therefore, it appears very promising to develop an optimized Grid Rearrangement-inspired 
algorithms for solving life-long \mpp problems. 
\end{remark}

\vskip 0.2in
\bibliography{sample}
\bibliographystyle{theapa}

\end{document}